%% file: main.tex
\documentclass[11pt]{article}

\usepackage[margin=1in]{geometry}

\usepackage[round]{natbib}
\usepackage{amsmath,amssymb,amsthm}
\usepackage{graphicx}
\usepackage{algorithm}
\usepackage{algpseudocode}
\usepackage{microtype}
\usepackage{hyperref}
\hypersetup{
  hidelinks,
  pdftitle={Curating Always-Loaded Context for LLM Agents: A Capacitated Assortment Model with Censored Feedback},
  pdfauthor={Zexuan Liu, Yuning Yang, Tiancheng Zhao}
}

\newtheorem{proposition}{Proposition}
\newtheorem{theorem}{Theorem}
\newtheorem{corollary}{Corollary}

\newtheorem{propositionA}{Proposition}[section]
\newtheorem{lemmaA}[propositionA]{Lemma}
\newtheorem{corollaryA}[propositionA]{Corollary}
\newtheorem{theoremA}[propositionA]{Theorem}
\theoremstyle{definition}

\newtheorem{assumption}{Assumption}
\newtheorem{assumptionA}[propositionA]{Assumption}
\theoremstyle{remark}

\newenvironment{assumptionprime}[1][]{%
  \edef\savedassumption{\arabic{assumption}}%
  \setcounter{assumption}{0}%
  \begin{assumption}[#1]}{\end{assumption}%
  \setcounter{assumption}{\savedassumption}}

\newcommand{\Reg}{\mathrm{Reg}}
\newcommand{\PFI}{P_{\mathrm{FI}}}
\newcommand{\cL}{\mathcal{L}}
\newcommand{\cN}{\mathcal{N}}
\newcommand{\Rz}{R_0}
\newcommand{\Bern}{\mathrm{Bern}}
\newcommand{\Bin}{\mathrm{Bin}}
\DeclareMathOperator{\sign}{sign}

\title{Curating Always-Loaded Context for LLM Agents:\\
A Capacitated Assortment Model with Censored Feedback}
\author{%
Zexuan Liu\textsuperscript{1}
\quad
Yuning Yang\textsuperscript{2}
\quad
Tiancheng Zhao\textsuperscript{3}\\[0.6em]
\normalsize\textsuperscript{1}University of Illinois at Urbana-Champaign\\
\normalsize\textsuperscript{2}Georgia Institute of Technology\\
\normalsize\textsuperscript{3}Saint Louis University}
\date{}

\begin{document}

\maketitle

\begin{abstract}

  At the start of every session, LLM agents load a fixed context file, such as \texttt{AGENTS.md}. Each loaded token in the file is charged again in every later round of the session, and these files can degrade performance as they grow in size. However, in practice, human or automated curators usually grow these files by appending.
  We formulate context curation as a capacitated assortment problem. Instructions consume tokens under a finite attention capacity; adding an instruction never raises the compliance of the others, while retained instructions incur a per-session setup cost.
  We prove an upper bound on the optimal file size, regardless of the number of available candidate instructions, and that appending every instruction with positive standalone value can be arbitrarily worse in net value than selecting an optimal subset. A token budget also limits the loss when the token price is underestimated.
  We then examine what can be learned from past sessions and how this information can guide decisions to add or remove instructions. Feedback is inherently censored: the benefits and harms of loaded instructions are observable, whereas missing instructions generate feedback only when their absence causes harm.
  In this setting, we show that deleting instructions ignored by agents can inevitably remove helpful ones. We characterize how much evidence should be collected before adding an instruction. Besides, we bound regret when human reviewers can inspect only a limited number of edits per period. Empirical experiments further show that irrelevant rules drawn from real context files reduce language-model compliance.

\end{abstract}

\section{INTRODUCTION}
\label{sec:intro}

A coding agent that works on a Python package reads the package's
\texttt{AGENTS.md} at the start of every session. The file holds
instructions such as ``run \texttt{pytest} after every edit'', ``add a
\texttt{CHANGELOG.md} line'', ``use NumPy docstrings'', ``never commit to
\texttt{main}'', and ``always use conda''. The last one breaks the
environment in a repository that uses \texttt{uv}. This example illustrates why instructions must be selected carefully: an unsuitable rule can lead the agent to make mistakes.

Even instructions that cause no harm incur a setup cost in each
session, regardless of their relevance to the current task. Context files can raise inference cost by
a fifth or more without a matching gain in task success
\citep{gloaguen2026agentsmd}, since
the file remains part of the prompt throughout the session, occupying context space in every model call. As the file grows,
the agent may follow its instructions less reliably. At high instruction density
the more instructions the file holds, the less reliably the agent follows
any one of them \citep{jaroslawicz2025ifscale,harada2025instructions}.
However, \citet{mcmillan2026adherence} and \citet{zhang2026guardrails}, studying real configuration files, observe no drop in adherence as file size grows within the small ranges they test.

In practice the file grows by appending. Across 1,867 repositories such
files more than double over their life, and deletions become rarer as the
file ages \citep{chakrabarti2026claudemd}. Tools that write the file from
past sessions add to the pressure. Every session is summarized into new
candidate instructions, so the pool of candidates grows without bound
while the agent's attention does not.

These lead to two questions. What should the file contain? And how can a
human or automated curator learn that from the transcripts of past
sessions? The second is harder than it looks, because transcripts are
censored. For instructions already in the file, transcripts reveal whether the agent complied with them and whether compliance helped or hurt. For instructions not in the file, their value is usually invisible; it becomes observable only when a failure reveals the missing instruction.

We answer the first question with an assortment model. Each candidate
instruction is an item with a setup cost, a probability of being
relevant, and a gain when followed. The gain can be negative as a followed
instruction can harm. The agent's limited attention is shared among all instructions. Adding an
instruction may lower the compliance of every instruction already in the
file, so the file is an assortment whose items compete for one resource.
Every loaded instruction carries a per-session setup cost. This is the
fixed-cost assortment problem of operations management
\citep{kunnumkal2019tractable,feldman2019relating} with a capacity
constraint \citep{rusmevichientong2010dynamic,desir2022capacitated}. Two
things are new: the items compete for a shared attention that dilutes as
the file grows, and the pool of items grows with time.

We answer the second question with an observation model that models the
censored feedback. We ask four questions: what censored transcripts identify, when a
deletion policy is safe, how much corroboration to require before adding
an instruction, and what a cap on human review costs.

To our knowledge, no prior work models instruction compliance as choice under a
finite attention capacity, and no existing learning result handles censored feedback under an edit
cap. Classical assortment models optimize revenue under a fixed choice model 
\citep{talluri2004revenue,berbeglia2020assortment}; we instead bound the size of
the optimal file under a attention-capacity bound. Learning with one-sided feedback
\citep{helmbold2000apple} and learning under a bounded number of switches
\citep{altschuler2021limited} have each been studied separately. The unique combination with a
growing candidate pool and a curator who can review only a limited number of edits has not been studied.

\paragraph{Contributions.}
(i)~A model of context curation as capacitated assortment problem with a
per-session setup cost, in which compliance is diluted by a shared
attention load (Section~\ref{sec:model}).
(ii)~A structural characterization under finite attention capacity, including a file-size bound independent of the number of candidate instructions and an unbounded accretion loss, which would be optimal without attention dilution (Section~\ref{sec:structure}).
(iii)~What censored feedback identifies: compliance and harm identify
different parameters, so deleting on non-compliance is value-blind and can
remove the better instruction (Section~\ref{sec:learning}).
(iv)~An evidence threshold for adding an instruction, and regret when a
reviewer accepts at most $m$ edits per period, with the edit-cap term
matched above and below (Section~\ref{sec:learning}).
(v)~Experiments in which the optimal file size stays flat as the
candidate pool grows from 200 to 3,200 items for two models, and a simulation matching the edit-cap
regret bounds (Section~\ref{sec:experiments}).

Section~\ref{sec:model} defines the model, Section~\ref{sec:structure}
gives the structural results, Section~\ref{sec:learning} develops the learning
results, and Section~\ref{sec:experiments} reports the experiments. Proofs are in
the appendix; Appendix~\ref{app:structure} also gives the results on
token budgets.
The numbered results, except those relying on the smoothness-to-bound
step in Assumption~\ref{ass:tied}, have complete formal proofs in Lean 4
(supplementary material).

\section{MODEL}
\label{sec:model}

\paragraph{The curation workflow.}
An agent runs sessions. After each session a summarizer, itself a
language model, reads the transcript and extracts candidate instructions:
things the agent should have known, or mistakes it should not repeat. The
candidates join a pool. A curator selects a subset
of the pool as the file that the agent loads at the start of every later
session. The pool grows with every session; the file should not.

\paragraph{Primitives.}
The pool is $[N] = \{1, \dots, N\}$. We view it as the first $N$ items of
an endless list of candidates, so that we can ask what happens as $N$
grows. Every bound below holds for all $N$. The file is a subset
$S \subseteq [N]$. Instruction $i$ has a \emph{token cost}
$c_i \ge c_{\min} > 0$, and $C(S) = \sum_{i \in S} c_i$ is the cost of the
file.

Instruction $i$ is \emph{relevant} in a session with probability $p_i$.
When it is relevant and the agent follows it, the session gains $r_i$ in
expectation. The gain is negative for a harmful instruction. Write
$w_i = p_i r_i$ for the per-session value of always following $i$, and
$\bar w = \max(0, \sup_i w_i) < \infty$, the supremum over all candidates.

Given the file $S$, the agent follows a relevant instruction $i \in S$
with probability $q_i(S) \in [0, 1)$, its \emph{compliance}. Compliance
depends on the whole file, not on $N$.

Every loaded instruction is charged once per session, when the session
starts and the file enters the prompt. Later turns of the session hit the
prompt cache. We call $\kappa > 0$ the \emph{token price}, so the
\emph{setup cost} of instruction $i$ is $\kappa c_i$ per session and the
file's setup cost is $\kappa C(S)$. The per-session value of $S$ is its
gross value minus its setup cost,
\begin{equation}
\label{eq:value}
R(S) = \Rz(S) - \kappa\, C(S), \qquad \Rz(S) = \sum_{i \in S} w_i\, q_i(S).
\end{equation}
The curator solves $\max\{R(S) : S \subseteq [N],\ C(S) \le B\}$ for a
budget $B \ge 0$. The empty file is feasible and $R(\emptyset) = 0$, so an
optimum $S^\star$ exists and $R(S^\star) \ge 0$.

An instruction with positive standalone value must be followed often enough to pay
its tokens. Call $\rho = \kappa c_{\min} / \bar w$ the \emph{break-even
compliance} ($\rho = \infty$ if $\bar w = 0$). If $R(\{i\}) > 0$, then
$q_i(\{i\}) > \rho$.

\paragraph{General compliance model.}
The compliance functions $q_i(S)$ are otherwise arbitrary, subject to the
following structural condition: adding instructions cannot increase the
compliance of an instruction already in the file.

\begin{assumption}[Monotone dilution]
\label{ass:mono}
$q_i(S \cup T) \le q_i(S)$ for every finite $S$, $T$ and every $i \in S$.
\end{assumption}

\paragraph{Shared-dilution subclass.}
The shared-dilution form is a parametric subclass of the general compliance
model. Each instruction has
an \emph{intrinsic compliance} $\theta_i > 0$ and an \emph{attention load}
$a_i > 0$; the file load is $A(S) = \sum_{j \in S} a_j$. A known
\emph{dilution curve} $g : [0, \infty) \to (0, 1]$, nonincreasing with
$g(0) = 1$, gives the common discount
\begin{equation}
\label{eq:sd}
q_i(S) = \theta_i\, g\big(A(S)\big), \qquad i \in S.
\end{equation}
Here $\theta_i$ measures how readily the agent follows an instruction, $a_i$
measures its attention load, and $g$ describes how attention is diluted. We
require $\theta_i g(a_i) < 1$ for every $i$. Since $g$ is nonincreasing,
\eqref{eq:sd} satisfies monotone dilution. It also gives the ratio
$q_i(S)/q_j(S) = \theta_i/\theta_j$, which does not depend on $S$, and a
single-item file has $q_i(\{i\}) = \theta_i g(a_i)$, which is the
compliance a curator observes when testing an instruction on its own.

Taking $S = \{i\}$, $q_i(S) \le q_i(\{i\})$ for every file that contains
$i$: no instruction is followed more often in company than alone. Under
\eqref{eq:sd}, equality is allowed when $g$ is flat on the added-load
interval; a strict drop occurs only when $g$ takes a smaller value.

\paragraph{Compliance mass and capacity.}
The total compliance of a file is $Q(S) = \sum_{i \in S} q_i(S)$: the
expected number of followed instructions when all are relevant. Its
maximum over files of $n$ items is the \emph{compliance mass}
$M(n) = \sup\{Q(S) : |S| = n\}$. We use two nested mass conditions.
The weaker one controls average compliance:

\begin{assumption}[Bounded average compliance]
\label{ass:bounded}
$\limsup_{n \to \infty} M(n)/n < \rho$.
\end{assumption}

The stronger condition is a uniform capacity bound:

\begin{assumption}[Attention capacity]
\label{ass:capacity}
$Q(S) \le \bar Q$ for every finite $S$, for a constant $\bar Q < \infty$.
\end{assumption}

The agent follows at most $\bar Q$ instructions per session on average,
however many it is given. Under~\eqref{eq:sd} the capacity is
$\bar Q = \chi \bar G$ with $\chi = \sup_i \theta_i / a_i$ and
$\bar G = \sup_{x \ge 0} x\, g(x)$, which is finite exactly when
$g(x) = O(1/x)$. Since $M(n) \le \bar Q$,
Assumption~\ref{ass:capacity} implies Assumption~\ref{ass:bounded}.
The main size bound uses capacity; the appendix also gives the weaker
result when capacity fails. Under the shared-dilution form,
Assumption~\ref{ass:bounded} holds whenever the
loads have a floor $a_i \ge a_{\min} > 0$ and
$\sup_i \theta_i \cdot \lim_{x \to \infty} g(x) < \rho$, in particular
whenever $g$ vanishes at infinity. The case
$\bar Q = 1$ is the single-choice case of classical assortment models:
the agent follows at most one instruction on average. Appendix~\ref{app:mnl}
records which classical choice models are instances of~\eqref{eq:sd}, and
the sharper constants they give. Assumption~\ref{ass:capacity} can fail
under~\eqref{eq:sd} while Assumption~\ref{ass:bounded} holds, for example
when $g$ decays like a power below one.

% \paragraph{Example.}
% In the file of Section~\ref{sec:intro} each instruction costs about $15$
% tokens. The \texttt{pytest} instruction is relevant in most sessions and
% has a positive gain; the changelog instruction is relevant only when
% behaviour changes; the conda instruction has a negative gain, so $w_i < 0$.
% Appending an instruction with load $a_j$ moves every other instruction's
% compliance from $\theta_i g(A(S))$ down to $\theta_i g(A(S) + a_j)$.
% Whether the appended instruction pays for itself is the question of the
% next section.

\section{STRUCTURAL PROPERTIES}
\label{sec:structure}

Each instruction incurs a setup cost, while the agent has a bounded total
compliance, so the value of a file cannot keep growing with its size. This section shows that the optimal file size is bounded independently of the candidate pool size. Adding every instruction that has positive standalone value can produce an arbitrarily large accretion loss. Without dilution, however, this policy is optimal whenever the token budget accommodates all such instructions. Proofs are in
Appendix~\ref{app:structure}.

\paragraph{The size bound.}
The bound is based on two observations. A file with nonnegative value earns at least its setup
cost: $R(S) \ge 0$ means $\kappa C(S) \le \sum_{i \in S} w_i q_i(S)$. And
each followed instruction is worth at most $\bar w$, while the agent
follows at most $\bar Q$ instructions per session
(Assumption~\ref{ass:capacity}). So the setup cost of a file worth
keeping is at most $\bar w \bar Q$, whatever the pool holds.

\begin{proposition}[File size bound]
\label{prop:size}
Under Assumption~\ref{ass:capacity}, every file with $R(S) \ge 0$, and
in particular every optimal file $S^\star$, satisfies
\[
\kappa\, C(S) \le \bar w\, \bar Q
\qquad\text{and}\qquad
|S| \le \frac{\bar w\, \bar Q}{\kappa\, c_{\min}} = \frac{\bar Q}{\rho}.
\]
Neither bound depends on the pool size $N$ or the budget $B$. If
$B \ge \bar w \bar Q / \kappa$, the budget never binds. Under~\eqref{eq:sd}
the capacity is $\bar Q = \chi \bar G$.
\end{proposition}

Appendix~\ref{app:structure} states the full version
(Proposition~\ref{prop:sizefull}), a bound under
Assumption~\ref{ass:bounded} when capacity fails, and a load bound on
optima that needs neither condition (Proposition~\ref{prop:load}).

The bound gives the curator a token budget that does not exclude any optimal file.
Suppose the best instruction in the pool is worth
$\bar w / \kappa = 10{,}000$ always-loaded tokens per session, and the
agent follows at most $\chi \bar G = 3$ instructions per session on average.
Then no file worth keeping costs more than $30{,}000$ tokens, and no
optimal file has more than $3/\rho$ instructions, however many sessions
have been summarized.

\paragraph{Accretion.}
A curator who tests each instruction alone and appends the ones with positive standalone value
is running \emph{accretion}. Write
$v_i^{\mathrm{solo}} = R(\{i\}) = w_i q_i(\{i\}) - \kappa c_i$ for the standalone value of
instruction $i$, which is $w_i \theta_i g(a_i) - \kappa c_i$
under~\eqref{eq:sd}. The accretion file is
$S_N^{\mathrm{acc}} = \{i \in [N] : v_i^{\mathrm{solo}} > 0\}$ and its closure is
$\overline{S}_N^{\mathrm{acc}} = \{i \in [N] : v_i^{\mathrm{solo}} \ge 0\}$. Every session adds candidates, and
every candidate with positive standalone value joins the file. But the file's
gross value is capped at $\bar w \chi \bar G$ while its setup cost grows with
every instruction. Past $\chi \bar G / \rho$ instructions, each addition increase the setup cost by
at least $\kappa c_{\min}$ per session.

\begin{corollary}[Accretion is unboundedly suboptimal]
\label{cor:accretion}
Assume~\eqref{eq:sd} with $\chi \bar G < \infty$, and let $S^\star_N$ be
optimal for the pool $[N]$. Every $S \subseteq [N]$ satisfies
$R(S) \le \bar w\, \chi \bar G - \kappa c_{\min} |S|$, hence
\[
R(S^\star_N) - R(S_N^{\mathrm{acc}}) \ \ge\ \kappa c_{\min}\, |S_N^{\mathrm{acc}}| - \bar w\, \chi \bar G,
\]
which tends to infinity whenever $|S_N^{\mathrm{acc}}| \to \infty$.
\end{corollary}

The loss grows linearly in the size of the accretion file, at rate
$\kappa c_{\min}$ per instruction beyond $\chi \bar G / \rho$. The hypothesis
$|S_N^{\mathrm{acc}}| \to \infty$ is met as soon as infinitely many candidates with positive standalone value, which the summarizer of
Section~\ref{sec:model} supplies. As the pool grows, so does the
number of instructions with positive standalone value. Appending all of them leaves an agent that follows about
$\chi \bar G$ of them per session while paying for every one.

\paragraph{Without dilution.}
In this case, the token price alone does not make accretion suboptimal. If compliance did
not depend on the file, value would be additive and accretion would be
optimal whenever the file fits within the token budget. This is the limit $a_j \to 0$
of~\eqref{eq:sd}, which lies outside the model, and it isolates dilution
as the cause of the loss above.

\begin{proposition}[Without dilution, accretion is optimal]
\label{prop:nodilution}
Suppose $q_i(S) = q_i$ for every $S \ni i$, so that
$v_i^{\mathrm{solo}} = w_i q_i - \kappa c_i$. Then $R(S) = \sum_{i \in S} v_i^{\mathrm{solo}}$. If $B \ge C(S_N^{\mathrm{acc}})$, then $S_N^{\mathrm{acc}}$ is
optimal, and the optimal files are exactly the $S$ with
$S_N^{\mathrm{acc}} \subseteq S \subseteq \overline{S}_N^{\mathrm{acc}}$ and $C(S) \le B$.
\end{proposition}

Without dilution, accretion is optimal whenever the token budget allows
it. Under monotone dilution, every optimal file contains only instructions
with nonnegative standalone value. At least one optimal file can be
obtained by deleting instructions from the accretion file
(Proposition~\ref{prop:inside} in the appendix). The next section studies
how a curator can use session feedback to identify which instructions to
delete. The appendix also gives a tighter size bound for files whose
value cannot be increased by deleting one instruction. A token budget
set at this bound limits the loss when the curator underestimates the
token price (Corollary~\ref{cor:insure}).

\section{LEARNING WITH CENSORED FEEDBACK}
\label{sec:learning}

Section~\ref{sec:structure} takes the values $w_i$ and the compliances
$q_i(S)$ as known. In this section, we investigate how curators can learn from session transcripts. Transcripts are censored: they show what the loaded file did,
not what an absent instruction would have done. Formal assumptions and
proofs are in Appendix~\ref{app:learning}.

\subsection{Setup}
\label{sec:setup}

Each session $t$ runs with the file $S_t$ loaded. A reviewer reads the
transcript. A curator edits the file. The reviewer approves at most $m$
edits per \emph{epoch} of $L$ sessions.

A transcript carries three kinds of signal.
\begin{itemize}\setlength{\itemsep}{0pt}
\item \emph{Loaded} $i \in S_t$: relevant? followed? a harm flag, raised
  with probability $\eta > 0$ when following $i$ did harm.
\item \emph{Absent} $i \notin S_t$: a gap sighting with probability
  $\delta_i p_i$, where $\delta_i$ is the reviewer's detection rate.
\item \emph{Observed gains}: the gain $X_{i,t} \in [-1, 1]$ of a followed
  instruction, shown only under \emph{magnitude feedback}.
\end{itemize}

\begin{assumption}[Censored feedback]
\label{ass:feedback}
In session $t$ the curator observes the first two signals above. All
draws are independent across instructions and sessions, and their law
depends on the file only through $q_i(S_t)$. The gains $r_i$ enter no
observation. Under magnitude feedback the curator also observes
$X_{i,t}$ for each followed $i \in S_t$.
\end{assumption}

\paragraph{What transcripts identify.}
Let $h_i$ be the probability that following $i$ does harm. The
\emph{typed} case is $h_i \in \{0, 1\}$: an instruction is always helpful
or always harmful. We prove what a fixed sequence of files identifies
(Theorem~\ref{thm:ident} in Appendix~\ref{app:ident}). The compliance
$q_i(S)$ of a loaded instruction is identified. The harm rate $\eta h_i$ is
identified, hence whether $i$ ever harms. The sign of the gain $r_i$ is
identified only in the typed case. The size of a gain is never identified
without magnitude feedback. An absent instruction is seen only through
its gap sightings, which measure a rate, not a value.

\subsection{Deletion}
\label{sec:deletion}

A natural policy deletes what the agent ignores. We prove that it is
blind to value.

\begin{proposition}[Deleting on non-compliance is value-blind]
\label{prop:blind}
Suppose Assumption~\ref{ass:feedback} holds, with a fixed file $S$ played in every
session. A deletion policy that uses only compliance data has the same
law of decisions under every parameter vector with the same $(p, q)$.
Take an instance of~\eqref{eq:sd} with $S = \{1, 2\}$. Instruction~1
is helpful but often ignored; instruction~2 is mostly followed but worth
little or harmful. Deleting the most-ignored instruction removes
instruction~1 with probability tending to one as sessions accumulate.
\end{proposition}

In the language of partial monitoring \citep{bartok2014partial}, the
compliance channel is feedback whose law does not separate the
hypotheses the curator must tell apart. Here, the compliance observed for an instruction also depends on the other instructions in the file through dilution. The feedback on $i$ is $q_i(S)$,
which the file $S$ sets through dilution.

\subsection{Admission}
\label{sec:admission}

An absent instruction becomes a candidate when reviewers keep sighting
its gap. A spurious candidate, one that names no real gap, is also
sighted now and then. Sightings are independent across sessions at a
fixed \emph{sighting rate}: $\varepsilon$ for a spurious candidate and
$\pi = \delta_j p_j$ for a real one (Assumption~\ref{ass:sight}). The
curator admits at the first session $\tau_k$ by which $k$ sightings
fall within a window of $W$ sessions. Over $T$ sessions, a spurious
candidate is falsely included with probability
$\PFI(k) = \Pr_\varepsilon(\tau_k \le T)$; a real one waits
$\mathbb E_\pi \tau_k$ sessions. With $N_s$ spurious candidates, a cost
$c_F$ per false inclusion and $c_D > 0$ per session of delay, the loss
of $k$ is $\cL(k) = \Lambda_F\, \PFI(k) + c_D\, \mathbb E_\pi \tau_k$
with $\Lambda_F = c_F N_s$. Write $\Lambda_D = c_D / \pi$ for the the expected delay cost until the first sighting of a real candidate. Let
$\psi_W = 1 - (1 - \pi)^{W - 1}$ be the chance that two real sightings
share a window. Let $x = \varepsilon W$ and $x_T = \varepsilon T$ count
the spurious sightings expected per window and per horizon, and let
$\omega = \Lambda_D / (\Lambda_F \varepsilon T)$.

\begin{theorem}[How many sightings before adding an instruction]
\label{thm:threshold}
Assumption~\ref{ass:sight} with $\pi, \varepsilon \in (0, 1]$, $W \ge 2$,
$\Lambda_F > 0$ and $T \ge 1$. If
\[
x < \omega\, (2 - 1/\psi_W) \qquad \text{and} \qquad
\omega / \psi_W < 1 - x_T / 2 - x,
\]
then $k = 2$ is the unique minimizer of $\cL$ over $k \ge 1$. As
$\psi_W \to 1$ the condition reads $x < \omega < 1 - x_T / 2 - x$.
\end{theorem}

Two sightings give the lowest expected loss whenever the price of one wait lies between the
false-inclusion costs that a third and a second sighting would prevent.

\subsection{Regret under an edit cap}
\label{sec:editcap}

Sessions run $t = 1, \dots, T$, with epoch $e(t) = \lceil t / L \rceil$.
The file starts empty, an edit is one addition or one removal, and each
epoch has at most $m$ edits. Every played file satisfies $C(S_t) \le B$.
A \emph{policy} maps the history of files and observations to the next
file. Its regret against the best fixed feasible file $S^\star$ is
$\Reg_T = \sum_{t \le T} (R(S^\star) - \mathbb E\, R(S_t))$, and
$\Delta_{\max} = \max_{C(S) \le B} [R(S^\star) - R(S)]$ is the largest
per-session loss. Without magnitude feedback every policy has linear
regret on some instance (Proposition~\ref{prop:magnitude}), so we assume
it from here on.

We analyze ETC-$m$ (Algorithm~\ref{alg:etc}), an explore-then-commit
policy satisfying the edit cap. It tries the candidates in $J = \lceil N / s
\rceil$ batches of $s = \min(\lfloor m/2 \rfloor, \lfloor B / c_{\max}
\rfloor)$ instructions, where $c_{\max} = \max_i c_i$, so that one
batch can be swapped for the next within one epoch's cap. Each batch is
played for $n_0$ sessions; then it commits to the file with the largest
estimated value. The estimator is in Appendix~\ref{app:regret}.

\begin{algorithm}[t]
\caption{ETC-$m$: explore-then-commit under an edit cap}
\label{alg:etc}
\begin{algorithmic}[1]
\Require pool $[N]$, cap $m$, epoch length $L$, budget $B$, exploration
length $n_0$
\State $s \gets \min(\lfloor m/2 \rfloor, \lfloor B / c_{\max} \rfloor)$;
partition $[N]$ into batches $\mathcal J_1, \dots, \mathcal J_J$ of at
most $s$ instructions
\For{$j = 1, \dots, J$} \Comment{explore}
  \State at the first session of an epoch, replace the loaded batch by
  $\mathcal J_j$ (at most $2s \le m$ edits)
  \State play $\mathcal J_j$ for $\lceil n_0 / L \rceil$ epochs; record
  compliance and realized gains
\EndFor
\State from the records, estimate the load $\hat a_i$ and the value
$\hat\beta_i$ of every instruction (Appendix~\ref{app:regret})
\State $\hat S \gets \arg\max \{ \hat R(S) : C(S) \le B \}$ with
$\hat R(S) = g(\hat A(S)) \sum_{i \in S} \hat\beta_i - \kappa C(S)$
\Comment{commit}
\State move from $\mathcal J_J$ to $\hat S$, removals first, $m$ edits
per epoch; then play $\hat S$ until $T$
\end{algorithmic}
\end{algorithm}
Predicting the value of unplayed files needs \emph{tied loads}
(Assumption~\ref{ass:tied}): $\theta_i = \lambda a_i$ with $\lambda$
known, and loads bounded by a known $a_{\max}$. Let $n_B$ be the largest
number of instructions a feasible file can hold, and $p_i \ge p_{\min} > 0$.

\begin{theorem}[ETC-$m$ under an edit cap]
\label{thm:etc}
Assumption~\ref{ass:feedback} with magnitude feedback and
compliance~\eqref{eq:sd} under tied loads; $m \ge 2$, $B \ge c_{\max}$.
For every $n_0 \ge 8 \ln(6NT) / p_{\min}$,
\[
\Reg_T \ \le\ \Delta_{\max}\, L \Big( J + \Big\lceil \frac{s + n_B}{m} \Big\rceil \Big)
\ +\ \Lambda_T(n_0),
\]
where the learning term $\Lambda_T(n_0)$ is given in
Theorem~\ref{thm:etcfull}.
\end{theorem}

\begin{theorem}[The edit-cap term is unavoidable]
\label{thm:lower}
Let $n_m = \lfloor N / m \rfloor$, $T \ge L n_m$, $N, m \ge 1$, $B > 0$.
There is a family of $N$ instances of Assumption~\ref{ass:feedback} with
magnitude feedback and compliance~\eqref{eq:sd}. They share
$c \le B$, $p$, $\theta$ and $a$, with noiseless gains and no harm. Under
instance $j$ the optimal file is $\{j\}$, missing $j$ costs at least
$\Delta = p\, \theta\, g(a) - \kappa c > 0$ per session, and
$\Delta_{\max} \le \Delta + \kappa B$. Every policy that adds at most $m$
instructions per epoch from the empty file satisfies
\[
\max_j \Reg_T(j) \ \ge\ \frac{\Delta L}{2} (n_m - 1)
\ \ge\ \frac{\Delta L N}{4 m} \quad (N \ge 4m).
\]
\end{theorem}

For $N \ge 4m$, $\kappa B \le \Delta$ and $m \le 2B / c_{\max}$, the
edit-cap terms of the two theorems match up to a constant factor. Both
are of order $(N / m) L \Delta$.

The learning term $\Lambda_T$ is of order $T^{2/3} (\ln NT)^{1/3}$, the
rate of explore-then-commit, and is not claimed tight
(Appendix~\ref{app:regret}). This is not the switch-budget bound of
\citet{simchilevi2023phase}, where a global budget leaves regret linear
in $T$. Here the cap renews each epoch, so every candidate is eventually
tried.

\input{experiments}

\section{RELATED WORK}
\label{sec:related}

\paragraph{Assortment with product costs and capacity.}
The nearest strand of operations management is assortment optimization
with a per-item cost \citep{kunnumkal2019tractable,feldman2019relating}
and with a cardinality or space constraint
\citep{rusmevichientong2010dynamic,desir2022capacitated}. Under the
choice model of Appendix~\ref{app:mnl} the optimal assortment is revenue-ordered
\citep{talluri2004revenue}; the costed and capped versions are hard, and
the papers above give relaxations, approximation schemes, and learning
policies. They bound revenue under a fixed choice model and supply
algorithms. We bound the size of the optimal file under an attention
capacity, with signed gains and a growing pool, and we study structure:
the size bound, accretion, and the no-dilution converse. Hardness and
approximation carry over from the cited work and are not our subject.

\paragraph{Other choice models.}
Shared dilution~\eqref{eq:sd} contains the general attraction model
\citep{gallego2015attraction}, in which an item's share is its attraction
over the sum of attractions plus an outside term. Monotone dilution (Assumption~\ref{ass:mono}) is what the
assortment literature calls regularity, which alone gives tight
guarantees for revenue-ordered sets \citep{berbeglia2020assortment} and
places a model among regular, random-utility, and consideration-set
classes \citep{gallego2024axioms}. In attention models from choice theory
the probability that an item is considered does not increase as the menu grows
\citep{lleras2017more,cattaneo2021overload}; the assumption there is
per-item, not a cap on the sum. Multipurchase models let the total mass
exceed one \citep{abdallah2026multipurchase}, as ours does, but with
additive payoffs. We know of no choice model that takes a capacity on
total choice probability as a primitive.

\paragraph{Learning with censored feedback and edit budgets.}
Our observation model descends from apple tasting
\citep{helmbold2000apple}, in which the label is seen only after the
positive action, and from partial monitoring \citep{bartok2014partial},
which classifies such problems by their observation map. What is new
here is that the map depends on the file through dilution. Bandits with a
global switch budget have a linear regret phase when the budget is below
the number of arms \citep{simchilevi2023phase,altschuler2021limited};
switch-efficient assortment bandits treat the switch count as an output
\citep{dong2020multinomial}; batched and streaming bandits limit when a
policy may change or how many arms fit in memory
\citep{perchet2016batched,he2025streaming}; apple tasting with a
switching cost has regret of order $\sqrt T$ \citep{cesari2026apple}. Our
cap differs from all of these: it is a per-period cap on edits that
renews, so the untried part of the pool shrinks at a fixed rate, and the
resulting edit-cap term is matched above and below.

\paragraph{LLM context and memory curation.}
Several recent systems edit files of the \texttt{AGENTS.md} kind from
past sessions: by probing and refining the file \citep{shepard2026probe},
under an edit budget with a token-cost pruner \citep{xia2026rrsi}, by
tracking which edit fixed or broke what \citep{ma2026sepo}, by compiling
trajectories into instructions behind replay gates
\citep{qiu2026autorefine}, or by turning user corrections into enforced
instructions \citep{zhou2026trace}. Memory curation has also been posed
as optimization: constrained retention \citep{kang2026retention},
promotion under a token budget \citep{annapureddy2026epistemics},
per-task budget allocation \citep{sharma2026limbo}, knapsack prompt
compression \citep{zhang2026submax}, value-minus-harm scores
\citep{wu2026forget}, and a bandit over memory operations
\citep{jiang2026memory}. All use additive or submodular value, under
which appending an item with positive marginal value is safe. Empirical
studies document the growth of these files
\citep{chakrabarti2026claudemd,chatlatanagulchai2025readmes}, their
per-session cost \citep{gloaguen2026agentsmd,lulla2026agentsmd}, and the
reduction in compliance as instructions accumulate
\citep{jaroslawicz2025ifscale,harada2025instructions,vasileva2026phase},
with no size effect at the sizes of real files
\citep{mcmillan2026adherence,zhang2026guardrails}. None models compliance
as choice: no item competes with the others for attention, no signal is
censored outside the file, and no result bounds regret under a review
cap.

\section{CONCLUSION}
\label{sec:conclusion}

% For a curator the results give four recommendations, which apply where
% attention is shared. Cap the file: under capacity a cap at
% $\bar w \bar Q / \kappa$ tokens never binds, whatever the pool holds
% (Proposition~\ref{prop:size}), and the appendix shows that a cap also
% bounds the loss when tokens are underpriced. Do not delete on
% non-compliance alone: compliance data cannot tell a helpful instruction
% from a harmful one (Proposition~\ref{prop:blind}), so delete on observed
% harm or on a measured value below the token cost. Require corroboration
% before adding: two sightings inside a window is the right threshold
% across a wide band of cost ratios (Theorem~\ref{thm:threshold}). Budget
% the review cap as a search cost: with $m$ edits per epoch of $L$
% sessions, finding the right instructions among $N$ candidates costs about
% $(N/m) L \Delta$, and no estimator avoids it
% (Theorems~\ref{thm:etc} and~\ref{thm:lower}). Two questions are open.
% The regret upper bound needs tied loads, and we do not know whether a
% family with item-specific dilution admits an estimability condition of
% this kind. The learning results await a setting with an edit cap; the regret
% experiment verifies the bounds on the theorem's instances only
% (Appendix~\ref{app:l3}).

We model context curation as an assortment problem with token costs, attention capacity, and censored feedback. Under our model, the optimal file size has a bound that does not depend on the candidate pool. As the number of instructions with positive standalone value grows without bound, appending all of them can become arbitrarily worse than optimal selection.
The learning analysis distinguishes compliance from value. Under our feedback model, compliance and harm observations do not identify gain magnitudes. We establish conditions under which two sightings within a time window minimize the loss from false admissions and delayed admissions. With measured gains and additional identification assumptions, we derive a learning guarantee under limits on file edits. The upper and lower bounds match the edit-related regret term up to constant factors.
Experiments with two language models show that large instruction files can reduce gross value, even before the token penalty. The best tested files remain small as the candidate pool grows under the chosen penalty. The published benchmark also shows lower average compliance at high instruction counts. The simulations reproduce the predicted scaling of regret caused by edit limits within our model.
An open question is which weaker conditions permit value identification when instructions interact in different ways. Another is whether curation based on measured task gains improves agent performance under actual limits on file edits.

\label{endofmain}

\bibliography{references}

\clearpage
\appendix

\section{PROOFS AND FURTHER RESULTS FOR SECTION 3}
\label{app:structure}

Throughout this appendix $\Theta(S) = \sum_{i \in S} \theta_i$,
$\bar\theta = \sup_i \theta_i$, $a_{\min} = \inf_i a_i$,
$g(\infty) = \lim_{x \to \infty} g(x)$, and for $y \in (0, 1]$,
$g^{-1}(y) = \sup\{x \ge 0 : g(x) \ge y\}$, which is finite iff
$y > g(\infty)$. Under~\eqref{eq:sd}, $Q(S) = \Theta(S)\, g(A(S))$.
Section~\ref{app:proofs3} proves the results of
Section~\ref{sec:structure}. Section~\ref{app:closure} shows that optimal
files lie inside the accretion closure. Section~\ref{app:bounded} gives
the size bound under Assumption~\ref{ass:bounded} when capacity fails,
and a load bound that needs neither mass condition.
Section~\ref{app:sharp} sharpens the bound on files that no single
deletion improves and shows that a cap at the sharpened bound insures
against an underestimated token price.

\subsection{Proofs for Section~\ref{sec:structure}}
\label{app:proofs3}

\begin{propositionA}[File size bound, full version]
\label{prop:sizefull}
Assume~\eqref{eq:sd} with $\chi \bar G < \infty$.
(i) Every $S$ satisfies
$R(S) \le \bar w\, \Theta(S)\, g(A(S)) - \kappa C(S)
\le \bar w\, \chi \bar G - \kappa C(S)$, where $\Theta(S) = \sum_{i \in S} \theta_i$.
(ii) Every $S$ with $R(S) \ge 0$ satisfies
$\kappa C(S) \le \bar w\, \chi \bar G$ and $|S| \le \chi \bar G / \rho$.
(iii) If $B \ge \bar w \chi \bar G / \kappa$, the token budget never binds: the capped
optima are exactly the uncapped optima.
(iv) Every optimum $S^\star$ has $R(S^\star) \ge 0$, so (ii) applies to it
for every $N$ and every $B$.
(v) In the single-choice case $\chi \bar G = 1$, every $S$ with $R(S) \ge 0$
has $C(S) \le \bar w / \kappa$ and $|S| \le 1/\rho$.
The proof uses~\eqref{eq:sd} only through Assumption~\ref{ass:capacity}
with $\bar Q = \chi \bar G$; Proposition~\ref{prop:size} is (ii)-(iv) so
stated.
\end{propositionA}

\begin{proof}[Proof of Propositions~\ref{prop:size} and~\ref{prop:sizefull}]
(i) Since $w_i \le \bar w$, $q_i(S) \ge 0$ and $\bar w \ge 0$,
$\Rz(S) \le \bar w\, Q(S) = \bar w\, \Theta(S)\, g(A(S))$. As
$\theta_i \le \chi a_i$ for every $i$, $\Theta(S) \le \chi A(S)$, so
$\Theta(S)\, g(A(S)) \le \chi\, A(S)\, g(A(S)) \le \chi \bar G$.
(ii) $R(S) \ge 0$ and (i) give $\kappa C(S) \le \bar w \chi \bar G$, and
$C(S) \ge c_{\min} |S|$ gives $|S| \le \bar w \chi \bar G / (\kappa c_{\min})
= \chi \bar G / \rho$. (If $\bar w = 0$ then $R(S) = -\kappa C(S) \ge 0$
forces $S = \emptyset$, and $\chi \bar G / \rho = 0$.)
(iii) Let $S_\infty$ be an uncapped optimum. It has
$R(S_\infty) \ge R(\emptyset) = 0$, so $C(S_\infty) \le \bar w \chi \bar G /
\kappa \le B$ by (ii): $S_\infty$ is feasible under the cap, hence a
capped optimum. Conversely a capped optimum $S$ has
$R(S) \ge R(S_\infty)$, so it is an uncapped optimum.
(iv) $\emptyset$ is feasible for every $B \ge 0$ and $R(\emptyset) = 0$;
the constants $\bar w$, $\chi$ and $\bar G$ are taken over all candidates.
(v) Put $\chi \bar G = 1$ in (ii).
\end{proof}

\begin{proof}[Proof of Corollary~\ref{cor:accretion}]
Proposition~\ref{prop:sizefull}(i) with $C(S) \ge c_{\min} |S|$ gives the
bound on $R(S)$. Apply it to $S = S_N^{\mathrm{acc}}$ and use
$R(S^\star_N) \ge 0$ (Proposition~\ref{prop:sizefull}(iv)).
\end{proof}

The hypothesis $|S_N^{\mathrm{acc}}| \to \infty$ is consistent with~\eqref{eq:sd} and
$\chi \bar G < \infty$: any pool with infinitely many candidates satisfying
$w_i \theta_i g(a_i) > \kappa c_i$ has it, for example identical
candidates with $\theta_i = a_i = 1$ and $w_i g(1) > \kappa c_i$.

\begin{proof}[Proof of Proposition~\ref{prop:nodilution}]
Additivity is immediate from $q_i(S) = q_i$. As $v_i^{\mathrm{solo}} > 0$ exactly on
$S_N^{\mathrm{acc}}$ and $v_i^{\mathrm{solo}} \le 0$ off it,
$R(S) \le \sum_{i \in S \cap S_N^{\mathrm{acc}}} v_i^{\mathrm{solo}} \le R(S_N^{\mathrm{acc}})$, with equality iff
$S_N^{\mathrm{acc}} \subseteq S$ and $v_i^{\mathrm{solo}} = 0$ on $S \setminus S_N^{\mathrm{acc}}$, that is
$S \subseteq \overline{S}_N^{\mathrm{acc}}$; $S_N^{\mathrm{acc}}$ is feasible because $B \ge C(S_N^{\mathrm{acc}})$.
\end{proof}

The budget condition is needed: with $g = (1, 1)$, $c = (1, 1)$ and
$B = 1$ the optima are the two singletons, and neither contains $S_N^{\mathrm{acc}}$.

\subsection{Optimal files lie inside the accretion closure}
\label{app:closure}

Under monotone dilution accretion keeps too much but never too little:
an instruction with $w_i \le 0$ is in no optimal file, and one with negative standalone value never appear in optimal file. The statement and
its proof hold for any positive price in place of $\kappa$ and for any
budget.

\begin{propositionA}[Optimal files lie inside the accretion closure]
\label{prop:inside}
Under Assumption~\ref{ass:mono}, with $w_i$ of any sign, every optimal
$S^\star$ satisfies $S^\star \subseteq \{i : w_i > 0\} \cap \overline{S}_N^{\mathrm{acc}}$, and
$S^\star \cap S_N^{\mathrm{acc}}$ is also optimal.
\end{propositionA}

\begin{proof}
Let $P = \{i : w_i > 0\}$; subsets of feasible files are feasible. Two
inequalities hold for every $S$ under Assumption~\ref{ass:mono}.
(L1) $R(S \cap P) \ge R(S) + \kappa C(S \setminus P)$: indeed
$R(S \cap P) - R(S) = \sum_{j \in S \cap P} w_j [q_j(S \cap P) - q_j(S)]
- \sum_{j \in S \setminus P} w_j q_j(S) + \kappa C(S \setminus P)$; the first
sum is $\ge 0$ by monotone dilution and $w_j > 0$, the second is $\le 0$
because $w_j \le 0$.
(L2) If $w_j \ge 0$ for all $j \in S$ and $i \in S$, then
$R(S) - R(S \setminus \{i\}) \le v_i^{\mathrm{solo}}$: indeed
$R(S) - R(S \setminus \{i\}) = w_i q_i(S)
+ \sum_{j \in S \setminus \{i\}} w_j [q_j(S) - q_j(S \setminus \{i\})]
- \kappa c_i$; each bracket is $\le 0$ and $w_i q_i(S) \le w_i q_i(\{i\})$.
Now let $S^\star$ be optimal. If $S^\star \not\subseteq P$, then
$C(S^\star \setminus P) > 0$ and (L1) gives
$R(S^\star \cap P) > R(S^\star)$, a contradiction; so $S^\star \subseteq P$.
If some $i \in S^\star$ had $v_i^{\mathrm{solo}} < 0$, (L2) would give
$R(S^\star \setminus \{i\}) > R(S^\star)$; so $S^\star \subseteq \overline{S}_N^{\mathrm{acc}}$.
Finally remove the items with $v_i^{\mathrm{solo}} = 0$ from $S^\star$ one at a time; each
intermediate set $S'$ satisfies $S' \subseteq P$, so (L2) gives
$R(S' \setminus \{i\}) \ge R(S')$, the value never drops, and the last set
is $S^\star \cap S_N^{\mathrm{acc}}$.
\end{proof}

Assumption~\ref{ass:mono} cannot be weakened to
$q_i(S) \le q_i(\{i\})$: take candidates $1, 2, 3$ with $w = (1, 10, 10)$,
$c_i = 1$, $\kappa = 1$ and $B = 3$, $q_1 = 0.1$ in every file, and
$q_2 = q_3 = 0.9$ in every file except $\{2, 3\}$, where both are $0.1$.
Then $g_1 = -0.9$, yet the unique optimum is $\{1, 2, 3\}$ with
$R = 15.1$: adding instruction $1$ raises the compliance of $2$ and $3$.

\subsection{Size bounds without capacity}
\label{app:bounded}

Proposition~\ref{prop:size} uses~\eqref{eq:sd} only through the capacity
$\bar Q = \chi \bar G$. When $\bar G = \infty$, for example when $g$ decays
like a power below one, two bounds remain. The first uses only
Assumption~\ref{ass:bounded} and does not require the shared-dilution
parameterization. Let
$\cN = \{n \ge 0 : \kappa c_{\min}\, n \le \bar w M(n)\}$ be the sizes at
which the compliance mass can still pay for the file, and
$M^*(n) = \max_{k \le n} M(k)$.

\begin{propositionA}[File size bound under bounded-average compliance]
\label{prop:sizegeneral}
(i) Every $S$ satisfies $R(S) \le \bar w M(|S|) - \kappa C(S)$.
(ii) If $R(S) \ge 0$, then $\kappa C(S) \le \bar w M(|S|)$ and $|S| \in \cN$.
(iii) Under Assumption~\ref{ass:bounded}, $\cN$ is finite. With
$\bar n = \max \cN$, every $S$ with $R(S) \ge 0$ has $|S| \le \bar n$ and
$C(S) \le \bar w M^*(\bar n) / \kappa$, and a cap
$B \ge \bar w M^*(\bar n) / \kappa$ never binds.
(iv) Every optimal $S^\star$ has $R(S^\star) \ge 0$, so (ii) and (iii)
apply to it, with bounds that depend on neither $N$ nor $B$.
\end{propositionA}

\begin{proof}
(i) Since $w_i \le \bar w$, $q_i(S) \ge 0$ and $\bar w \ge 0$,
$\Rz(S) \le \bar w\, Q(S) \le \bar w M(|S|)$.
(ii) Combine $R(S) \ge 0$ with (i) and $C(S) \ge c_{\min} |S|$.
(iii) If $\bar w = 0$, then $\cN = \{0\}$. Otherwise
Assumption~\ref{ass:bounded} gives $M(n) < \rho\, n$, that is
$n \notin \cN$, for all large $n$. Then
$\kappa C(S) \le \bar w M(|S|) \le \bar w M^*(\bar n)$. The cap never
binds by the argument of Proposition~\ref{prop:sizefull}(iii), with
$\bar w M^*(\bar n)$ in place of $\bar w \chi \bar G$.
(iv) $\emptyset$ is feasible and $R(\emptyset) = 0$; $\bar w$ and $M$ are
taken over all candidates.
\end{proof}

\begin{lemmaA}
\label{lem:mono}
Under Assumption~\ref{ass:mono}, $M(n+1)/(n+1) \le M(n)/n$ for $n \ge 1$.
Hence Assumption~\ref{ass:bounded} holds iff $M(n_1) < \rho\, n_1$ for some
$n_1 \ge 1$, and $\cN$ is an initial segment of $\mathbb N$.
\end{lemmaA}

\begin{proof}
Fix $S'$ with $|S'| = n + 1$ and let $j \in S'$ minimize $q_j(S')$. Removing
the smallest of $n + 1$ terms keeps at least a share $n/(n+1)$ of the sum,
and Assumption~\ref{ass:mono} gives $q_i(S' \setminus \{j\}) \ge q_i(S')$,
so
$M(n) \ge \sum_{i \in S' \setminus \{j\}} q_i(S' \setminus \{j\})
\ge \sum_{i \in S' \setminus \{j\}} q_i(S')
\ge \frac{n}{n+1} \sum_{i \in S'} q_i(S')$.
Take the supremum over $S'$. A nonincreasing sequence has limit superior
equal to its infimum, which is below $\rho$ iff some term is. For
$n \ge 1$, $n \in \cN$ iff $M(n)/n \ge \rho$, which is preserved as $n$
decreases, and $0 \in \cN$.
\end{proof}

\paragraph{Remarks.}
Without Assumption~\ref{ass:mono}, $\cN$ need not be an interval
(Lemma~\ref{lem:mono} needs the assumption), which is why $\bar n$ is
defined as a maximum: take identical candidates with
$w_i = c_i = 1$, $\kappa = 0.5$, and $q_i(S) = 0.1$ if $|S| \le 2$, $1$ if
$|S| = 3$, and $0$ if $|S| \ge 4$; then $\cN = \{0, 3\}$ and $R(S) = 1.5$
for $|S| = 3$. The strict inequality in Assumption~\ref{ass:bounded}
cannot be relaxed: for identical instructions with
$q(n) = \rho + (1 - \rho)/n^2$ and $0 < \rho < 1$, Assumption~\ref{ass:mono}
holds, $M(n)/n \to \rho$, every instruction is in $S_N^{\mathrm{acc}}$, and
$R(S) = \bar w (1 - \rho)/|S|$, so $R(S_N^{\mathrm{acc}}) \to 0$ while
$R(S^\star_N) = \bar w (1 - \rho)$ for $B \ge c_{\min}$.

\paragraph{Examples.}
Table~\ref{tab:examples} takes identical instructions ($w_i = \bar w > 0$,
$c_i = c_{\min}$) whose compliance $q(n)$ in a file of $n$ is
nonincreasing. Then Assumption~\ref{ass:mono} holds, $M(n) = n q(n)$,
$R(S) = \bar w\, [M(n) - \rho\, n]$ for $|S| = n$, and
$\bar n = \max \{n : q(n) \ge \rho\}$ ($0$ if $q(1) < \rho$). When $M$ is
bounded, the cost bound $\bar w \sup_n M(n) / \kappa$ is linear in
$\bar w / \kappa$. Under power dilution ($0 < \beta < 1$), $\bar n$ grows
like $(\bar w / \kappa)^{1/\beta}$ and, for $\alpha > \rho$, blows up as
$\beta \downarrow 0$, the no-dilution limit.

\begin{table}[h]
\caption{Identical instructions with compliance $q(n)$ in a file of $n$,
  where $0 < \alpha \le 1$, $b, u > 0$, $0 < \beta < 1$, $u = v / v_0$ and
  $\lfloor y \rfloor^+ = \max(0, \lfloor y \rfloor)$.
  Assumption~\ref{ass:bounded} holds in the first three rows for every
  $\rho$, and in the last iff $\alpha < \rho$. The $\bar n$ column assumes
  $\bar w > 0$; if $\bar w = 0$, then $\cN = \{0\}$ and $\bar n = 0$.}
\label{tab:examples}
\centering
\begin{tabular}{llll}
Dilution & $q(n)$ & $M(n)$ & $\bar n$ \\
\hline
Single choice & $\frac{u}{1 + nu}$ & $< 1$ & $\lfloor 1/\rho - 1/u \rfloor^+$ \\
Exponential & $\alpha e^{-bn}$ & $\le \frac{\alpha}{be}$ & $\lfloor \ln(\alpha/\rho)/b \rfloor^+$ \\
Power ($\beta < 1$) & $\alpha n^{-\beta}$ & $\alpha n^{1-\beta}$ & $\lfloor (\alpha/\rho)^{1/\beta} \rfloor$ \\
None & $\alpha$ & $\alpha n$ & $\infty$ if $\alpha \ge \rho$, else $0$ \\
\end{tabular}
\end{table}

If the budget does not bind, the optimal size maximizes $M(n) - \rho\, n$,
which is unimodal in each row, so it is $\lfloor n^* \rfloor$ or
$\lceil n^* \rceil$ for the continuous maximizer $n^*$, clipped at $0$:
$n^* = (\sqrt{u/\rho} - 1)/u$ (single choice), $(1 - W_0(e \rho / \alpha))/b$ with
$W_0$ the principal Lambert branch (exponential),
$((1 - \beta) \alpha / \rho)^{1/\beta}$ (power), and $\infty$ if
$\alpha > \rho$ (none). In the power row the optimal size is therefore about
a fraction $n^* / \bar n \approx (1 - \beta)^{1/\beta} < 1/e$ of $\bar n$ as
$\rho \to 0$. Power dilution with $\beta \ge 1$ is excluded from the table:
then $M(n) = \alpha n^{1-\beta}$ is nonincreasing for $n \ge 1$, so the
optimal size is $1$ if $\alpha > \rho$, $0$ if $\alpha < \rho$, and both
$0$ and $1$ are optimal if $\alpha = \rho$.

The second bound applies to optima only and uses the form~\eqref{eq:sd}
without capacity or bounded average compliance: every instruction in an optimal
file is followed often enough to cover its setup cost at the file's attention load,
which bounds the load.

\begin{propositionA}[Load bound on optima]
\label{prop:load}
Assume~\eqref{eq:sd} with $\bar w > 0$ and $\bar\theta < \infty$, and let
$S^\star$ be a nonempty optimum under any budget. Then $w_i > 0$ on
$S^\star$, $\kappa c_i \le w_i \theta_i\, g(A(S^\star))$ for every
$i \in S^\star$, and $g(A(S^\star)) \ge \rho / \bar\theta$. Hence, if
$\rho / \bar\theta > 1$ the only optimum is $\emptyset$; and if
$g(\infty) < \rho / \bar\theta \le 1$ then
$A(S^\star) \le g^{-1}(\rho / \bar\theta) < \infty$ and, when
$a_{\min} > 0$, $|S^\star| \le g^{-1}(\rho / \bar\theta) / a_{\min}$.
\end{propositionA}

\begin{proof}
By Assumption~\ref{ass:mono}, $w_i > 0$ on $S^\star$ by
Proposition~\ref{prop:inside}. Fix $i \in S^\star$; the file
$S^\star \setminus \{i\}$ is feasible. By Assumption~\ref{ass:mono} and
$w_l > 0$ for $l \ne i$,
$R(S^\star \setminus \{i\}) \ge \sum_{l \in S^\star \setminus \{i\}} w_l q_l(S^\star)
- \kappa C(S^\star) + \kappa c_i = R(S^\star) - w_i q_i(S^\star) + \kappa c_i$,
and optimality gives $w_i q_i(S^\star) \ge \kappa c_i$, that is
$\kappa c_i \le w_i \theta_i\, g(A(S^\star))$. Since $w_i \le \bar w$,
$\theta_i \le \bar\theta$ and $c_i \ge c_{\min}$,
$g(A(S^\star)) \ge \kappa c_{\min} / (\bar w \bar\theta) = \rho / \bar\theta$.
As $g \le 1$, this is impossible when $\rho / \bar\theta > 1$. When
$g(\infty) < \rho / \bar\theta \le 1$, $A(S^\star)$ belongs to
$\{x : g(x) \ge \rho / \bar\theta\}$, whose supremum $g^{-1}(\rho / \bar\theta)$
is finite; and $|S^\star|\, a_{\min} \le A(S^\star)$.
\end{proof}

\subsection{Sharpened size bound and cap insurance}
\label{app:sharp}

Proposition~\ref{prop:size} bounds every file with $R(S) \ge 0$. Files
that no single deletion improves obey a sharper bound, because the
deletion conditions bound the total setup cost by the file's own gross
value. Throughout this subsection~\eqref{eq:sd} holds. For a price
$\mu \ge 0$ let $R_\mu(S) = \Rz(S) - \mu C(S)$, so $R = R_\kappa$, and let
$R^\star_\mu = \max_{S \subseteq [N]} R_\mu(S)$ be the uncapped optimum at
price $\mu$. A nonempty file $S$ is \emph{deletion-stable at $\mu$} if
$R_\mu(S \setminus \{i\}) \le R_\mu(S)$ for every $i \in S$. Every nonempty
optimum under any budget is deletion-stable at $\kappa$, since deleting
an item keeps the file feasible, and so is every nonempty file returned
by a local search with deletion moves.

For a nonempty $S$ write $A = A(S)$, $W(S) = \sum_{i \in S} w_i \theta_i$,
$\phi = \Rz(S) = g(A)\, W(S)$, $a^S_{\max} = \max_{i \in S} a_i$, and
\[
\begin{aligned}
e_i(S) &= \frac{1}{\theta_i}\Big(\frac{1}{g(A)} - \frac{1}{g(A - a_i)}\Big) \ge 0, \\
E(S) &= \sum_{i \in S} \theta_i\, e_i(S), \\
X_g(S) &= \Theta(S)\, g(A - a^S_{\max})\, \big[1 - g(A) E(S)\big].
\end{aligned}
\]
The quantity $e_i(S)$ is nonnegative because $g$ is nonincreasing and
$a_i > 0$. Let $\bar X_g = \sup X_g(S)$ over nonempty finite $S$ with
$w_i > 0$ on $S$ ($\bar X_g = 0$ if there is no such $S$), and
$a_g = \bar w \bar X_g$.

\begin{lemmaA}[Single deletion]
\label{lem:single}
For nonempty $S$ and $i \in S$,
$\Rz(S) - \Rz(S \setminus \{i\}) = \theta_i\, g(A - a_i)\, [w_i - \phi\, e_i(S)]$.
\end{lemmaA}

\begin{proof}
$\Rz(S) - \Rz(S \setminus \{i\}) = g(A) W(S) - g(A - a_i)(W(S) - w_i \theta_i)
= w_i \theta_i\, g(A - a_i) - W(S)[g(A - a_i) - g(A)]$. By the definition
of $e_i$, $g(A - a_i) - g(A) = \theta_i e_i(S)\, g(A)\, g(A - a_i)$, so the
last term is $\phi\, \theta_i e_i(S)\, g(A - a_i)$.
\end{proof}

\begin{propositionA}[Sharpened size bound]
\label{prop:sharp}
Assume~\eqref{eq:sd}, let $w_i$ be of any sign, and let $S$ be
deletion-stable at $\mu > 0$.
(i) $w_i > \phi\, e_i(S)$ for every $i \in S$, and
$W(S)\, [1 - g(A) E(S)] > 0$. If $g(A) E(S) < 1$, in particular if
$|S| = 1$, then $\phi > 0$ and $w_i > 0$ on $S$.
(ii) If $\phi > 0$, then $w_i > 0$ on $S$ and
\[
\mu\, C(S) \le \phi \cdot \frac{g(A - a^S_{\max})}{g(A)}\, \big[1 - g(A) E(S)\big]
\le \phi, \qquad \text{so } R_\mu(S) \ge 0 .
\]
(iii) If $\phi > 0$, then $\mu\, C(S) \le \bar w\, X_g(S) \le a_g$, and
$X_g(S) \le Q(S)$; under $\chi \bar G < \infty$,
$a_g \le \bar w \chi \bar G$.
Hence every optimum at any price $\mu > 0$ under any budget has
$\mu\, C(S^\star) \le a_g$, and a cap $B \ge a_g / \kappa$ never binds:
its optima are exactly the uncapped optima.
\end{propositionA}

\begin{proof}
(i) Stability at $\mu$ and Lemma~\ref{lem:single} give
$\theta_i\, g(A - a_i)\, [w_i - \phi e_i(S)] \ge \mu c_i > 0$, and
$\theta_i g(A - a_i) > 0$, so $w_i > \phi e_i(S)$. Multiply by $\theta_i$
and sum over $S$: $W(S) > \phi E(S) = g(A) W(S) E(S)$, that is
$W(S)[1 - g(A)E(S)] > 0$. If $1 - g(A)E(S) > 0$ then $W(S) > 0$, so
$\phi = g(A) W(S) > 0$ and $w_i > \phi e_i(S) \ge 0$. For $S = \{i\}$,
$E(S) = 1/g(a_i) - 1$ and $g(A)E(S) = 1 - g(a_i) < 1$.
(ii) If $\phi > 0$ then $W(S) > 0$, so $1 - g(A)E(S) > 0$ by (i), and then
$w_i > \phi e_i(S) \ge 0$. Sum the stability inequalities over $S$, using
$g(A - a_i) \le g(A - a^S_{\max})$ (as $A - a_i \ge A - a^S_{\max}$ and $g$
is nonincreasing) and the positive brackets:
$\mu C(S) \le g(A - a^S_{\max}) \sum_{i \in S} \theta_i [w_i - \phi e_i(S)]
= g(A - a^S_{\max}) [W(S) - \phi E(S)]$, and $W(S) = \phi / g(A)$ gives the
first bound. For the second, every term of $E(S)$ is nonnegative, so
$E(S) \ge 1/g(A) - 1/g(A - a^S_{\max})$, hence
$1 - g(A)E(S) \le g(A)/g(A - a^S_{\max})$ and the first bound is at most
$\phi$.
(iii) $W(S) \le \bar w\, \Theta(S)$ since $w_i \le \bar w$, $\theta_i > 0$
and $\bar w \ge 0$. Insert this in
$\mu C(S) \le g(A - a^S_{\max}) [W(S) - \phi E(S)] = g(A - a^S_{\max})\, W(S)\,
[1 - g(A)E(S)]$, where the bracket is positive by (ii), to get
$\mu C(S) \le \bar w X_g(S)$; by (ii) $w_i > 0$ on $S$, so
$X_g(S) \le \bar X_g$. The inequality
$g(A - a^S_{\max})[1 - g(A)E(S)] \le g(A)$ of (ii) gives
$X_g(S) \le \Theta(S) g(A) = Q(S)$, and $Q(S) \le \chi \bar G$ as in
Proposition~\ref{prop:sizefull}(i).
For the consequences, a nonempty optimum $S^\star$ at price $\mu$ is
deletion-stable at $\mu$ and has $w_i > 0$ on $S^\star$ by
Proposition~\ref{prop:inside} at price $\mu$, so $\phi > 0$ and (iii)
gives $\mu C(S^\star) \le a_g$; the empty optimum has cost $0$. The cap
never binds by the argument of Proposition~\ref{prop:sizefull}(iii), with
$a_g$ in place of $\bar w \chi \bar G$.
\end{proof}

The case split in (i) is needed. Take $g(x) = e^{-x}$, two identical
instructions with $\theta_i = a_i = 1$, $w_i = -1$, $c_i = 1$, and
$\mu = 0.01$. Then $R_\mu(\{1, 2\}) = -2e^{-2} - 0.02 \approx -0.291$
exceeds $R_\mu(\{1\}) = -e^{-1} - 0.01 \approx -0.378$, so $\{1, 2\}$ is
deletion-stable with $\phi < 0$, both $w_i < 0$ and $R_\mu < 0$; here
$g(A) E(S) = 2(1 - e^{-1}) > 1$. Deleting a harmful instruction
un-dilutes the other harmful instruction. Optima are unaffected, since
they have $\phi > 0$. Under the model of Appendix~\ref{app:mnl},
$g(A)E(S) = Q(S) < 1$ for every file, so the sign conclusions of (i) hold for every deletion-stable
file, and $X_g(S)$ has the explicit form of Appendix~\ref{app:mnl}; a
closed form of $\bar X_g$ for a general $g$ is open, and the bound
$a_g \le \bar w \chi \bar G$ is what we use.

A cap also protects a curator who misestimates the token price. One who treats tokens
as free, as an appending tool does, can lose without bound; a cap at
$a_g / \kappa$ bounds that loss by $a_g$, and when the price is
overestimated the cap never binds, so it adds no loss of its own. The cap
must be set by whoever estimates the token price more accurately than the curator; set with
the curator's own underestimate it protects nothing.

\begin{lemmaA}[Envelope bound]
\label{lem:envelope}
Assume~\eqref{eq:sd}. For $0 < \kappa \le \mu$,
$R^\star_\kappa - R^\star_\mu \le a_g \ln(\mu / \kappa)$.
\end{lemmaA}

\begin{proof}
For $0 < \mu_1 \le \mu_2$ let $S_1$ be an uncapped optimum at $\mu_1$. Then
$R^\star_{\mu_1} = R_{\mu_2}(S_1) + (\mu_2 - \mu_1) C(S_1)
\le R^\star_{\mu_2} + (\mu_2 - \mu_1) a_g / \mu_1$, because
$\mu_1 C(S_1) \le a_g$ by Proposition~\ref{prop:sharp} applied at price
$\mu_1$ ($S_1$ is an optimum at $\mu_1$, or empty). Fix $n \ge 1$ and put
$\mu_j = \kappa (\mu/\kappa)^{j/n}$ for $j = 0, \dots, n$. Summing the
one-step bound over the $n$ steps,
$R^\star_\kappa - R^\star_\mu \le a_g \sum_{j < n} (\mu_{j+1}/\mu_j - 1)
= a_g\, n\, ((\mu/\kappa)^{1/n} - 1)$, which tends to
$a_g \ln(\mu / \kappa)$ as $n \to \infty$.
\end{proof}

\begin{corollaryA}[A cap insures against an underestimated token price]
\label{cor:insure}
Assume~\eqref{eq:sd}. Let the curator solve
$\max\{\Rz(S) - \hat\kappa C(S) : S \subseteq [N],\ C(S) \le B\}$ exactly
with an estimate $\hat\kappa \ge 0$ of the true price $\kappa > 0$, and
call $R^\star_\kappa - R(S)$ its loss.
(i) If $\hat\kappa \le \kappa$ and $B \ge a_g / \kappa$, the loss is at
most $(\kappa - \hat\kappa) \min(B, a_g / \hat\kappa)$, with
$a_g / 0 = \infty$; at $B = a_g / \kappa$ this is at most
$a_g (1 - \hat\kappa/\kappa) \le a_g$.
(ii) If $\hat\kappa \ge \kappa$ and $B \ge a_g / \kappa$, the cap is
slack ($C(S) \le a_g / \hat\kappa$, and $S$ is an uncapped optimum at
$\hat\kappa$), and the loss is at most
$\min\{a_g \ln(\hat\kappa / \kappa),\ R^\star_\kappa\}$.
(iii) Without a cap and with $0 < \hat\kappa \le \kappa$, the loss is at
most $a_g (\kappa / \hat\kappa - 1)$.
\end{corollaryA}

\begin{proof}
Let $S$ be the curator's file, an optimum of the capped problem at price
$\hat\kappa$, and $S_\kappa$ an uncapped optimum at $\kappa$. When
$\hat\kappa > 0$ and $S \ne \emptyset$, $S$ is deletion-stable at
$\hat\kappa$ with $w_i > 0$ on $S$ (Proposition~\ref{prop:inside} at price
$\hat\kappa$), so $\hat\kappa C(S) \le a_g$ by
Proposition~\ref{prop:sharp}(iii); the same holds for $S = \emptyset$.
For every file $S'$,
\begin{equation}
\label{eq:telescope}
R(S') - R(S) = [R_{\hat\kappa}(S') - R_{\hat\kappa}(S)]
+ (\kappa - \hat\kappa)(C(S) - C(S')).
\end{equation}
(i) $C(S_\kappa) \le a_g / \kappa \le B$ by
Proposition~\ref{prop:sharp}, so $S_\kappa$ is feasible for the curator and
$R_{\hat\kappa}(S_\kappa) \le R_{\hat\kappa}(S)$. With $S' = S_\kappa$ and
$\hat\kappa \le \kappa$, \eqref{eq:telescope} gives
$R^\star_\kappa - R(S) \le (\kappa - \hat\kappa) C(S)$. Now $C(S) \le B$,
and if $\hat\kappa > 0$ then $C(S) \le a_g / \hat\kappa$; hence
$C(S) \le \min(B, a_g / \hat\kappa)$. At $B = a_g / \kappa$ the bound
is $(\kappa - \hat\kappa) a_g / \kappa = a_g (1 - \hat\kappa / \kappa)$.
(ii) $C(S) \le a_g / \hat\kappa$. An uncapped optimum $S_{\hat\kappa}$ at
$\hat\kappa$ has $C(S_{\hat\kappa}) \le a_g / \hat\kappa \le a_g / \kappa
\le B$, so it is feasible for the curator and
$R_{\hat\kappa}(S_{\hat\kappa}) \le R_{\hat\kappa}(S)$: $S$ is an uncapped
optimum at $\hat\kappa$, that is $R_{\hat\kappa}(S) = R^\star_{\hat\kappa}$.
Since $\hat\kappa \ge \kappa$, $R(S) = R_{\hat\kappa}(S) + (\hat\kappa -
\kappa) C(S) \ge R^\star_{\hat\kappa} \ge R_{\hat\kappa}(\emptyset) = 0$, so
the loss is at most $R^\star_\kappa$, and it is at most
$R^\star_\kappa - R^\star_{\hat\kappa} \le a_g \ln(\hat\kappa / \kappa)$
by Lemma~\ref{lem:envelope}.
(iii) Without a cap, $S_\kappa$ is feasible, and \eqref{eq:telescope} with
$S' = S_\kappa$ gives $R^\star_\kappa - R(S) \le (\kappa - \hat\kappa) C(S)
\le (\kappa - \hat\kappa) a_g / \hat\kappa = a_g (\kappa/\hat\kappa - 1)$.
\end{proof}

Under the model of Appendix~\ref{app:mnl} the bound in (iii) is attained
for every $\hat\kappa$, so the uncapped loss is unbounded
as $\hat\kappa \to 0$; whether it is attained for a general $g$ is open.

\section{THE MULTINOMIAL LOGIT CASE}
\label{app:mnl}

The multinomial logit model is shared dilution with $g(x) = 1/(1+x)$ and
$\theta_i = a_i = u_i = v_i / v_0$.

\begin{assumptionprime}[Multinomial logit]
\label{ass:mnl}
Instruction $i$ has weight $v_i > 0$, the agent's default behaviour has
weight $v_0 > 0$, and $q_i(S) = v_i / (v_0 + V(S))$ with
$V(S) = \sum_{i \in S} v_i$.
\end{assumptionprime}

Write $U(S) = \sum_{i \in S} u_i$ and $D(S) = 1 + U(S)$, so that
$q_i(S) = u_i / D(S)$ and $Q(S) = U(S)/D(S) = 1 - 1/D(S) < 1$. Every
instance satisfies the requirement $\theta_i g(a_i) = u_i/(1 + u_i) < 1$
of Section~\ref{sec:model}. Here $\chi = 1$ and $\bar G = \sup_x x/(1+x) = 1$,
so the capacity is $\bar Q = 1$: the agent follows at most one
instruction per session on average, which is the single-choice case of
Proposition~\ref{prop:sizefull}(v). The quantities of
Appendix~\ref{app:sharp} read, with $u_{\max} = \max_{i \in S} u_i$:
$W(S) = \sum_{i \in S} w_i u_i$, $\phi = \Rz(S) = W(S)/D(S)$,
$e_i(S) = 1$, $E(S) = U(S)$, $g(A)E(S) = Q(S) < 1$, and
\[
\Rz(S) - \Rz(S \setminus \{i\}) = \frac{u_i (w_i - \phi)}{D(S) - u_i},
\qquad
X_g(S) = \frac{U(S)}{D(S)\, (D(S) - u_{\max})} .
\]
Since $g(A)E(S) < 1$ for every file, Proposition~\ref{prop:sharp}(i)
gives $w_i > \phi > 0$ on every deletion-stable file, and
Proposition~\ref{prop:sharp}(ii) reads
$\mu C(S) \le \phi / (D(S) - u_{\max})$.

The capped problem is capacitated multinomial logit assortment with
product costs \citep{kunnumkal2019tractable}, NP-hard in general
\citep{desir2022capacitated}.

\section{PROOFS AND FURTHER RESULTS FOR SECTION 4}
\label{app:learning}

Section~\ref{app:assumptions} states the latent-variable model behind
Assumption~\ref{ass:feedback} and the tied-loads condition.
Section~\ref{app:ident} states and proves Theorem~\ref{thm:ident}.
Section~\ref{app:deletion} proves Proposition~\ref{prop:blind} and
gives the result on magnitude feedback. Section~\ref{app:admission} proves
Theorem~\ref{thm:threshold}. Section~\ref{app:regret} gives the
estimator of ETC-$m$, its error bound, the full regret bound with its
tuned form, and the proofs of Theorems~\ref{thm:etc} and~\ref{thm:lower}.

\subsection{Formal assumptions}
\label{app:assumptions}

\begin{assumptionA}[Observations]
\label{ass:obs}
The parameters are $\varpi = (p, q, h, r^+, r^-, \eta, \delta)$:
relevance $p_i \in [0, 1]$; a compliance function $q$ on pairs $(S, i)$
with $i \in S$, with values in $[0, 1]$; harm probability
$h_i \in [0, 1]$; gain $r_i^+ \ge 0$ given a helpful outcome and loss
$r_i^- \ge 0$ given a harmful one, so $r_i = (1 - h_i) r_i^+ - h_i r_i^-$;
harm-flag rate $\eta \in (0, 1]$; gap-detection rates
$\delta_i \in (0, 1]$. A \emph{design} is a fixed sequence of files
$(S_t)_{t \ge 1}$. In session $t$ each item $i$ has independent latent
draws $\xi_{i,t} \sim \Bern(p_i)$ (relevance),
$F_{i,t} \sim \Bern(q_i(S_t))$ (compliance), $\zeta_{i,t} \sim \Bern(h_i)$
(harmful outcome), $\vartheta_{i,t} \sim \Bern(\eta)$ (flag) and
$\gamma_{i,t} \sim \Bern(\delta_i)$ (gap detection), independent across
items and sessions, and a gain $X_{i,t}$ with
$\mathbb E[X_{i,t} \mid \zeta_{i,t} = 0] = r_i^+$ and
$\mathbb E[X_{i,t} \mid \zeta_{i,t} = 1] = -r_i^-$, independent of the other
draws given $\zeta_{i,t}$; the gains are independent across items and
sessions and, given $\zeta_{i,t}$, identically distributed. The curator
observes, for $i \in S_t$: $\xi_{i,t}$; if $\xi_{i,t} = 1$, $F_{i,t}$; if
$\xi_{i,t} F_{i,t} = 1$, the harm flag $H_{i,t} = \zeta_{i,t} \vartheta_{i,t}$.
For $i \notin S_t$ it observes only the gap flag
$\Gamma_{i,t} = \xi_{i,t} \gamma_{i,t}$. It never observes $X_{i,t}$.
\end{assumptionA}

The expected per-session gain of a loaded item is
$p_i q_i(S) r_i = w_i q_i(S)$, as in~\eqref{eq:value}. Under~\eqref{eq:sd}
the compliance function is $q_i(S) = \theta_i g(A(S))$ and its parameters
are $(\theta, a)$. The \emph{typed} case is $h_i \in \{0, 1\}$ with
$r_i^\pm > 0$.

\begin{assumptionA}[Sightings]
\label{ass:sight}
A candidate's sighting indicators $Y_t \in \{0, 1\}$, $t \ge 1$, are
i.i.d.\ $\Bern(\pi)$, with $\pi = \varepsilon$ for a spurious candidate
and $\pi = \delta_j p_j$ for a real one.
\end{assumptionA}

\begin{assumptionA}[Magnitude feedback]
\label{ass:mag}
In addition to Assumption~\ref{ass:obs}: for $i \in S_t$ with
$\xi_{i,t} F_{i,t} = 1$ the curator observes $X_{i,t} \in [-1, 1]$, so
$|w_i| \le 1$; $p_i \ge p_{\min} > 0$ for all $i$; and the curator has an
exact, measurable oracle for the capped problem with estimated
parameters (the maximization in ETC-$m$ below).
\end{assumptionA}

\begin{assumptionA}[Tied loads]
\label{ass:tied}
Compliance follows~\eqref{eq:sd} with $\theta_i = \lambda a_i$ for a
known $\lambda > 0$, the curve $g$ is known and continuously
differentiable, and $a_i \le a_{\max}$ for a known $a_{\max}$. With
$s$ the batch size of ETC-$m$, $x_s = s a_{\max}$ and $G(x) = x g(x)$,
\[
\gamma_s = \min_{x \in [0, x_s]} G'(x) > 0 .
\]
\end{assumptionA}

Assumption~\ref{ass:feedback} is Assumption~\ref{ass:obs}, and with
magnitude feedback it is Assumption~\ref{ass:mag}. Under tied loads
$Q(S) = \lambda A(S) g(A(S)) = \lambda G(A(S))$, and $\gamma_s > 0$ makes
$G$ injective on $[0, x_s]$, the range of the load of a batch. This is
an identification condition, not a regularity condition: for
$g(x) = e^{-\beta x}$, $G$ peaks at $1/\beta$, and if $x_s > 1/\beta$
two loads $A_1 < 1/\beta < A_2$ with $G(A_1) = G(A_2)$ give a batch the
same compliances under $a^{(1)}$ and under
$a^{(2)}_i = a^{(1)}_i A_2 / A_1$, while the two instances disagree on
larger files.

\paragraph{Observable rates.}
Under Assumption~\ref{ass:obs} the observations of different items and
sessions are independent, so the law of the observation process under a
design is the product of the laws of the single observations $O_{i,t}$.
For $i \in S_t$, $O_{i,t}$ takes the values $\xi = 0$,
$(\xi, F) = (1, 0)$, $(\xi, F, H) = (1, 1, 0)$ and $(1, 1, 1)$ with
probabilities $1 - p_i$, $p_i (1 - q_i(S_t))$,
$p_i q_i(S_t) (1 - \eta h_i)$ and $p_i q_i(S_t)\, \eta h_i$. For
$i \notin S_t$, $O_{i,t} = \Gamma_{i,t} \sim \Bern(\delta_i p_i)$. Hence two
parameter vectors give the same observation law iff, for every $(t, i)$
with $i \in S_t$, they agree on $p_i$, $p_i q_i(S_t)$ and
$p_i q_i(S_t) \eta h_i$, and for every $(t, i)$ with $i \notin S_t$ they
agree on $\delta_i p_i$. We call these the \emph{observable rates}. They
do not involve $(r^+, r^-)$, and they involve $q$ only through the
values $q_i(S_t)$ on played pairs.

\subsection{Identification}
\label{app:ident}

A parameter is \emph{identified} under a design if any two parameter
vectors with the same law of the observations agree on it.

\begin{theoremA}[What censored transcripts identify]
\label{thm:ident}
Assumption~\ref{ass:feedback} with compliance~\eqref{eq:sd}, under a
design $(S_t)$.
(a) For $i \in S_t$ with $p_i > 0$, $q_i(S_t)$ is identified, and nothing
else about compliance is: two compliance functions that agree on the
played pairs $(S_t, i)$ give the same law. The law of the compliance data
depends on the parameters only through $(p, q)$, so no test of $r_i > 0$
against $r_i < 0$ that uses compliance data alone has power above its
size.
(b) If $\{i\}$ is played and $p_i > 0$, then $\theta_i g(a_i)$ is
identified. A design of singleton files identifies neither $\theta_i$ nor
$a_i$. Within one played file, $\theta_i / \theta_j$ is identified when
$p_i, p_j > 0$.
(c) If $S$ and $S \cup \{i\}$ are played, $i \notin S$, and some $k \in S$
has $p_k > 0$, then $g(A(S) + a_i) / g(A(S))$ is identified, and $a_i$
is identified once $A(S)$ is, provided the map
$x \mapsto g(A(S) + x)$ is injective on the allowed interval for $a_i$.
$A(S)$ is identified when $p_j > 0$ and
$\theta_j = \lambda a_j$ on $S$ with $\lambda$ known, and $x \mapsto x g(x)$
is injective on an interval known to contain $A(S)$.
(d) For $i \in S_t$ with $p_i q_i(S_t) > 0$, $\eta h_i$ is identified,
hence whether $i$ ever harms. In the typed case $\sign(r_i)$ is
identified; in general it is not, and $|r_i|$ never is.
(e) For $i$ never played, the observations involving $i$ are sightings
with law $\Bern(\delta_i p_i)$ per session, and nothing separates
$\delta_i$ from $p_i$. If the sighting rate $\delta$ is common and some
$j_0$ with $p_{j_0} > 0$ is played in one session and left out in
another, then $\delta$ and every $p_i$ are identified.
\end{theoremA}

Part (a) is the key fact: compliance carries no information about value,
however many sessions are observed. Testing an instruction alone measures
$\theta_i g(a_i)$, but not what the instruction takes from the others.
Separating the load $a_i$ from $\theta_i$ needs two files that differ in
$i$ and a known load for the smaller one; the tied-loads condition
supplies that load. Harm flags, unlike compliance, carry the sign of
value in the typed case (d). One instruction played both in and out of
the file calibrates the sightings of all the others (e).

\begin{proof}[Proof of Theorem~\ref{thm:ident}]
(a) If $p_i > 0$, the observable rates give $q_i(S_t) = p_i q_i(S_t)/p_i$.
Conversely the rates involve $q$ only through the played pairs, so two
compliance functions that agree there give the same law. The compliance
data $(\xi_{i,t}, F_{i,t})$ have the product law with parameters
$(p_i, q_i(S_t))$, which depend only on $(p, q)$. Given $\varpi$, let
$\varpi'$ keep $(p, q, \eta, \delta)$ and set $h_i' = 0$,
$r_i^{+\prime} = 1$, $r_i^{-\prime} = 0$ and, for $j \ne i$, $h_j' = 1$,
$r_j^{+\prime} = 0$, $r_j^{-\prime} = 1$; then $r_i' = 1 > 0 > -1 = r_j'$,
and $\varpi, \varpi'$ have the same compliance law; swapping the roles
gives a $\varpi''$ with $r_i'' < 0$. A test based on compliance data has
the same rejection probability under $\varpi$ and under a $\varpi'$ of
the opposite sign, so its power equals its size.
(b) By (a), $q_i(\{i\}) = \theta_i g(a_i)$ is identified. If only
singleton files are played, item $i$ enters the observable rates only
through $\theta_i g(a_i)$; for any $a' > 0$ the pair
$(\theta_i g(a_i) / g(a'),\, a')$ gives the same value, satisfies the
requirement $\theta' g(a') = \theta_i g(a_i) < 1$, and changes no other
rate, so neither $\theta_i$ nor $a_i$ is identified. Within a played
file $S \ni i, j$ with $p_i, p_j > 0$, $q_i(S)/q_j(S) = \theta_i/\theta_j$
by~\eqref{eq:sd}.
(c) By (a), $q_k(S)$ and $q_k(S \cup \{i\})$ are identified, and their
ratio is $g(A(S) + a_i) / g(A(S))$. Given $A(S)$, the injectivity condition
in (c) makes $a_i$ the unique solution. Under the stated condition every
$q_j(S)$, $j \in S$, is
identified, hence $Q(S) = \sum_{j \in S} \lambda a_j g(A(S)) = \lambda G(A(S))$
with $G(x) = x g(x)$, and $A(S)$ is the unique preimage of $Q(S)/\lambda$
on the known interval.
(d) If $p_i q_i(S_t) > 0$, the observable rates give $\eta h_i$, hence
$\mathbf 1\{h_i > 0\}$ as $\eta > 0$. In the typed case $h_i \in \{0, 1\}$
is read off, and $\sign(r_i) = +1$ if $h_i = 0$ and $-1$ if $h_i = 1$,
since $r_i^\pm > 0$. In general, $r_i = r_i^+ (1 - h_i (1 + \lambda_i))$
with $\lambda_i = r_i^-/r_i^+$ when $r_i^+ > 0$; the observation law does
not depend on $(r^+, r^-)$, and for $h_i \in (0, 1)$ the map
$\lambda \mapsto 1 - h_i(1 + \lambda)$ takes both signs on $[0, \infty)$,
so choosing $r_i^{+} = 1$ and $\lambda_i$ on either side of
$(1 - h_i)/h_i$ gives two parameter vectors with the same law and
opposite signs. Since no observable rate involves $(r^+, r^-)$, $|r_i|$
is not identified.
(e) For an item never played, the only observations involving it are the
$\Gamma_{i,t}$, i.i.d.\ $\Bern(\delta_i p_i)$, and no observable rate of another
item involves $(\delta_i, p_i)$; so the law depends on $(\delta_i, p_i)$
only through the product, and scaling $\delta_i$ by $c$ and $p_i$ by
$1/c$, for $c$ near $1$, leaves it unchanged. With common $\delta$ and a
bridge item $j_0$, the session with $j_0 \in S_t$ identifies $p_{j_0}$
(rate $p_{j_0}$), the session with $j_0 \notin S_t$ identifies
$\delta p_{j_0}$, so $\delta = \delta p_{j_0} / p_{j_0}$; then every
out-of-file $p_i = \delta p_i / \delta$, and every in-file $p_i$ is an
observable rate.
\end{proof}

\subsection{Deletion and the need for magnitudes}
\label{app:deletion}

\begin{proof}[Proof of Proposition~\ref{prop:blind}]
\emph{Instance and claims.} Take $q_1(S) = 0.1$, $q_2(S) = 0.8$,
$q_1(\{1\}) = 1/2$, $q_2(\{2\}) = 8/9$, $p_1 = p_2 = 1/2$,
$\kappa c_1 = \kappa c_2 = 0.001$, $r_1 = 1$ and $r_2 = \pm 0.05$, so that
$R(\{1\}) = 0.249 > R(\{2\}) \in \{0.021, -0.023\}$. We show: with
empirical non-compliance over each instruction's first $n$ relevant
sessions, the policy that deletes the most-ignored instruction, and the
policy that deletes every instruction above a threshold
$t_0 \in (0.2, 0.9)$, both delete instruction~1 and keep instruction~2
with probability tending to one as $n \to \infty$; with $r_2 = -0.05$ they
leave a file worth less than the empty one; and in the typed case with
$r_2 = -0.05$, the policy that deletes at the first harm flag never
deletes instruction~1 and deletes instruction~2 within $n'$ sessions with
probability $1 - (1 - 0.4\eta)^{n'}$.

\emph{Law of the decision.} By Theorem~\ref{thm:ident}(a) the compliance data of the first $t$
sessions have the same law under all parameter vectors with the same
$(p, q)$, so a measurable function of them has the same law. The
$\varpi'$ of that proof has $r_i' > 0$ and $r_j' < 0$ for $j \ne i$.
The instance satisfies~\eqref{eq:sd}: with $g(x) = 1/(1 + x)$ and
$\theta = a = (1, 8)$, $q_1(\{1\}) = 1/2$, $q_2(\{2\}) = 8/9$,
$q_1(S) = 1/10$, $q_2(S) = 8/10$, and $\theta_2 g(a_2) = 8/9 < 1$. Its
values are $R(\{1, 2\}) = 0.05 \pm 0.02 - 0.002$,
$R(\{1\}) = 0.5 \cdot 0.5 - 0.001 = 0.249$ and
$R(\{2\}) = \pm 0.05 \cdot 0.5 \cdot 8/9 - 0.001 \in \{0.021, -0.023\}$.
Under Assumption~\ref{ass:obs} the latent compliance draws $(F_{i,t})_t$
of item $i$ are i.i.d.\ $\Bern(q_i(S))$ and independent of its relevance
draws $(\xi_{i,t})_t$, so conditional on the relevance process, the draws
at the first $n$ relevant sessions are i.i.d.\ $\Bern(q_i(S))$ (each item
reaches $n$ relevant sessions almost surely since $p_i > 0$), and
Hoeffding's inequality gives, for the empirical compliance $\hat q_i$ over
those sessions, $\Pr(\hat q_i \ge q_i(S) + \epsilon) \le e^{-2n\epsilon^2}$
and $\Pr(\hat q_i \le q_i(S) - \epsilon) \le e^{-2n\epsilon^2}$. The argmax
policy keeps item~1 only if $\hat q_1 \ge \hat q_2$, which requires
$\hat q_1 \ge 0.45$ or $\hat q_2 \le 0.45$, each of probability at most
$e^{-2n(0.35)^2}$; so it deletes item~1 and keeps item~2 with probability
at least $1 - 2e^{-2n(0.35)^2} \to 1$. The threshold policy keeps item~1
iff $\hat q_1 \ge 1 - t_0 = 0.1 + (0.9 - t_0)$, of probability at most
$e^{-2n(0.9 - t_0)^2}$, and deletes item~2 iff
$\hat q_2 < 1 - t_0 = 0.8 - (t_0 - 0.2)$, of probability at most
$e^{-2n(t_0 - 0.2)^2}$; both tend to $0$. With $r_2 = -0.05$ both policies
end at $\{2\}$ with probability tending to $1$, and
$R(\{2\}) = -0.023 < 0 = R(\emptyset)$. For the harm-triggered policy in
the typed case with $h_1 = 0$, $h_2 = 1$: item~1 is never flagged, since
$\zeta_{1,t} = 0$; item~2 is flagged in a session with $S = \{1, 2\}$ with
probability $p_2 q_2(S) \eta = 0.4 \eta$, independently across sessions, so
it is deleted within $n'$ sessions with probability
$1 - (1 - 0.4\eta)^{n'}$.
\end{proof}

The next result is the reason Section~\ref{sec:editcap} assumes
magnitude feedback. Its setting is that of Section~\ref{sec:editcap}: a
policy is a probability kernel from the history of files and
observations to the next file, and given the history through session
$t - 1$ and $S_t$, the latent draws of session $t$ are independent of the
past and follow Assumption~\ref{ass:obs} with file $S_t$.

\begin{propositionA}[Magnitudes are necessary]
\label{prop:magnitude}
Assumption~\ref{ass:obs}, one candidate ($N = 1$), $\kappa > 0$,
$c_1 > 0$, $B \ge c_1$, $p_1 q_1(\{1\}) > 0$. There are parameter vectors
$\varpi, \varpi'$ with the same observation law such that $\emptyset$ is
optimal under $\varpi$ with margin
$\Delta = R_\varpi(\emptyset) - R_\varpi(\{1\}) > 0$, $\{1\}$ is optimal
under $\varpi'$ with margin
$\Delta' = R_{\varpi'}(\{1\}) - R_{\varpi'}(\emptyset) > 0$, and every
policy and every horizon $T$ have
$\max(\Reg_T(\varpi), \Reg_T(\varpi')) \ge T \min(\Delta, \Delta') / 2$.
\end{propositionA}

\begin{proof}
Take the typed helpful case ($h_1 = 0$) and two gains
$r_{\mathrm{lo}} < \kappa c_1 / (p_1 q_1(\{1\})) < r_{\mathrm{hi}}$; let
$\varpi$ have $r_1^+ = r_{\mathrm{lo}}$ and $\varpi'$ have
$r_1^+ = r_{\mathrm{hi}}$, all else equal. No observation depends on
$r_1^+$, so the two observation laws coincide, and a policy induces the
same law of $(S_t)_{t \le T}$ under both; in particular
$T_1 = \mathbb E \sum_{t \le T} \mathbf 1\{S_t = \{1\}\}$ is the same. Under
$\varpi$, $R(\{1\}) = -\Delta < 0 = R(\emptyset)$ with
$\Delta = \kappa c_1 - p_1 q_1(\{1\}) r_{\mathrm{lo}} > 0$, so
$S^\star = \emptyset$ and $\Reg_T(\varpi) = \Delta T_1$. Under $\varpi'$,
$R(\{1\}) = \Delta' > 0$ with
$\Delta' = p_1 q_1(\{1\}) r_{\mathrm{hi}} - \kappa c_1$, so
$S^\star = \{1\}$ (feasible as $B \ge c_1$) and
$\Reg_T(\varpi') = \Delta' (T - T_1)$. The larger of the two is at least
their average,
$(\Delta T_1 + \Delta'(T - T_1))/2 \ge \min(\Delta, \Delta') T/2$.
\end{proof}

\subsection{Admission}
\label{app:admission}

The admission policy of Section~\ref{sec:admission} admits at
$\tau_k = \inf\{t \ge 1 : \sum_{s = t - W + 1}^{t} Y_s \ge k\}$, with
$Y_s = 0$ for $s \le 0$.

\begin{lemmaA}[Binomial tail]
\label{lem:binom}
$\Pr(\Bin(n, y) \ge j) \le \binom{n}{j} y^j \le (ny)^j / j!$.
\end{lemmaA}

\begin{proof}
If at least $j$ of $n$ independent trials succeed, some $j$-subset of the
trials succeeds entirely; a union bound over the $\binom nj$ subsets gives
the first inequality, and $\binom nj \le n^j/j!$ the second.
\end{proof}

\begin{lemmaA}[False inclusion]
\label{lem:fi}
Under Assumption~\ref{ass:sight} with rate $\varepsilon$, $k \ge 1$,
$W \ge 1$: (1) $\PFI(1) = 1 - (1 - \varepsilon)^T$; (2)
$\PFI(k) \le T \varepsilon \Pr(\Bin(W - 1, \varepsilon) \ge k - 1)
\le T \varepsilon \binom{W-1}{k-1} \varepsilon^{k-1}
\le T \varepsilon ((W - 1)\varepsilon)^{k-1} / (k-1)!$.
\end{lemmaA}

\begin{proof}
(1) $\tau_1$ is the first sighting. (2) The event $\tau_k = t$ requires
$Y_t = 1$ and at least $k - 1$ sightings among the at most $W - 1$ sessions
$t - W + 1, \dots, t - 1$, whose count is independent of $Y_t$ and
stochastically at most $\Bin(W - 1, \varepsilon)$; so
$\Pr(\tau_k = t) \le \varepsilon \Pr(\Bin(W-1, \varepsilon) \ge k - 1)$. A union
bound over $t \le T$ and Lemma~\ref{lem:binom} finish.
\end{proof}

\begin{lemmaA}[Delay]
\label{lem:delay}
Under Assumption~\ref{ass:sight} with rate $\pi \in (0, 1]$: (1)
$\mathbb E \tau_k \ge k/\pi$ for all $k, W$; (2) for $W \ge 2$,
$\mathbb E \tau_2 = \frac1\pi (1 + 1/\psi_W)$ with
$\psi_W = 1 - (1 - \pi)^{W-1}$; and $\mathbb E \tau_1 = 1/\pi$.
\end{lemmaA}

\begin{proof}
(1) The policy cannot fire before the $k$-th sighting, whose time is a sum
of $k$ i.i.d.\ $\mathrm{Geom}(\pi)$ waiting times with mean $k/\pi$.
(2) Let $Z_1, Z_2, \dots$ be the i.i.d.\ $\mathrm{Geom}(\pi)$ times between
sightings ($Z_1$ is the first sighting time). The window count changes only
at sightings, and at the $j$-th sighting the window holds two sightings iff
$Z_j \le W - 1$. So the policy fires at the $j^*$-th sighting,
$j^* = \min\{j \ge 2 : Z_j \le W - 1\}$, and $\tau_2 = Z_1 + \dots + Z_{j^*}$.
Here $j^* - 1 \sim \mathrm{Geom}(\psi_W)$ with $\psi_W = \Pr(Z \le W - 1)$,
and $j^*$ is a stopping time for $(Z_j)_{j \ge 2}$ with finite mean, so
Wald's identity gives
$\mathbb E \tau_2 = \mathbb E Z_1 + \mathbb E[j^* - 1]\, \mathbb E Z
= 1/\pi + 1/(\psi_W \pi)$. Finally $\tau_1 = Z_1$.
\end{proof}

\begin{proof}[Proof of Theorem~\ref{thm:threshold}]
Write $\beta_\varepsilon = 1 - (1 - \varepsilon)^{W - 1}$ for the chance
that a spurious sighting has another inside its window. We first show that
$k = 2$ is the unique minimizer under the two conditions
\begin{align*}
\text{(A)}\ & \Lambda_F \big[1 - (1 - \varepsilon)^T - T \varepsilon\, \beta_\varepsilon\big]
  > \Lambda_D / \psi_W, \\
\text{(B)}\ & \Lambda_F\, T \varepsilon\, \beta_\varepsilon < \Lambda_D (2 - 1/\psi_W),
\end{align*}
which need only $\varepsilon \in [0, 1]$ and $\Lambda_F \ge 0$, and then
that the band implies them.
By Lemmas~\ref{lem:fi}(1) and~\ref{lem:delay},
$\cL(1) = \Lambda_F (1 - (1 - \varepsilon)^T) + \Lambda_D$. By
Lemma~\ref{lem:fi}(2) with $\Pr(\Bin(W-1, \varepsilon) \ge 1) =
\beta_\varepsilon$ and Lemma~\ref{lem:delay}(2),
$\cL(2) \le \Lambda_F T \varepsilon\, \beta_\varepsilon
+ \Lambda_D (1 + 1/\psi_W)$. By Lemma~\ref{lem:delay}(1),
$\cL(k) \ge \Lambda_D k \ge 3\Lambda_D$ for $k \ge 3$. Condition (A) says
exactly that the bound on $\cL(2)$ is below $\cL(1)$, and (B) that it is
below $3\Lambda_D$; so $\cL(2) < \cL(k)$ for every $k \ne 2$, $k \ge 1$.
For the band, use $(1 - \varepsilon)^T \le e^{-x_T}$ and
$\beta_\varepsilon \le (W - 1)\varepsilon \le x$: (A) follows from
$\Lambda_F (1 - e^{-x_T} - x_T x) > \Lambda_D / \psi_W$, and (B) from
$\Lambda_F x_T x < \Lambda_D (2 - 1/\psi_W)$. Dividing by
$\Lambda_F \varepsilon T > 0$, the second is $x < \omega (2 - 1/\psi_W)$, and
the first follows from $1 - e^{-x_T} \ge x_T - x_T^2/2$ and
$1 - x_T/2 - x > \omega/\psi_W$.
\end{proof}

In the regime $W\pi \gg k$ the delay is about $k/\pi$. When false inclusion
is rare, Lemma~\ref{lem:fi}(2) bounds $\PFI(k)$ by about
$\frac{T}{W} \frac{x^k}{(k-1)!}$, and counting only the $T/W$ disjoint
windows gives the lower estimate $\frac{T}{W} \frac{x^k}{k!}$, a factor $k$
smaller.

\subsection{Regret under an edit cap}
\label{app:regret}

\paragraph{Constants.}
Let $c_{\max} = \max_i c_i$,
$s = \min(\lfloor m/2 \rfloor, \lfloor B / c_{\max} \rfloor) \ge 1$ and
$J = \lceil N / s \rceil$. Let $\bar G = \sup_{x \ge 0} G(x)$ and
$B' = \min(B, \lambda \bar G / \kappa)$, read as $B$ when
$\bar G = \infty$; by Proposition~\ref{prop:sizefull}(ii), (iv) with
$\chi = \lambda$ and $\bar w \le 1$, some optimal file has $C(S^\star) \le B'$.
Let $\mathcal S' = \{S \subseteq [N] : C(S) \le B'\}$ and
$n_B = \max\{|S| : S \in \mathcal S'\}$. With $x_s = s a_{\max}$ and
$x_B = n_B a_{\max}$, put $g_s = g(x_s)$,
\[
L_{1/g} = \max_{x \in [0, x_s]} \frac{|g'(x)|}{g(x)^2},
\qquad
K_g = \sup_{x \ne y \in [0, x_B]} \frac{x\, |g(x) - g(y)|}{|x - y|}
\ \le\ x_B \max_{[0, x_B]} |g'| ,
\]
so that $1/g$ is $L_{1/g}$-Lipschitz on $[0, x_s]$. Finally
$\epsilon_0 = \sqrt{\ln(6NT) / (n_0 p_{\min})}$ and
\[
\varepsilon_{\mathrm{SD}}(n_0) = \epsilon_0 \Big[ \frac{2 + K_g}{g_s}
+ \frac{(1 + K_g)\, L_{1/g}\, s}{\lambda \gamma_s} \Big] .
\]

\paragraph{ETC-$m$.}
Fix an exploration length $n_0$, in sessions per batch. Partition $[N]$
into batches $\mathcal J_1, \dots, \mathcal J_J$ of at most $s$ items.
For $j = 1, \dots, J$, play $\mathcal J_j$ for $\lceil n_0 / L \rceil$
epochs, switching at the first session of an epoch ($\le 2s \le m$
edits); let $I_j$ be the set of sessions in which $\mathcal J_j$ is
played, $|I_j| = n \ge n_0$. For $i \in \mathcal J_j$ let $\hat q_i$ be
its empirical compliance over its relevant sessions in $I_j$ and
$\bar Y_i$ the mean over $I_j$ of $Y_{i,t} = \xi_{i,t} F_{i,t} X_{i,t}$.
Estimate the batch load and the item parameters by
\[
\hat A_j = G^{-1}\Big(\mathrm{clip}_{[0, G(x_s)]}
\Big(\frac{1}{\lambda} \sum_{l \in \mathcal J_j} \hat q_l\Big)\Big),
\qquad
\hat a_i = \mathrm{clip}_{[0, a_{\max}]}\Big(\frac{\hat q_i}{\lambda\, g(\hat A_j)}\Big),
\qquad
\hat\beta_i = \frac{\bar Y_i}{g(\hat A_j)},
\]
where $G^{-1}$ is the inverse of $G$ on $[0, x_s]$. For $S \in \mathcal S'$
let $\hat A(S) = \sum_{i \in S} \hat a_i$ and
\[
\hat R(S) = g(\hat A(S)) \sum_{i \in S} \hat\beta_i - \kappa C(S).
\]
Let
$\hat S$ maximize $\hat R$ over $\mathcal S'$. Move from
$\mathcal J_J$ to $\hat S$, removals first, $m$ edits per epoch at the
first session of each epoch, and play $\hat S$ thereafter.

The targets of the estimator are $\beta_i = w_i \theta_i = \lambda w_i a_i$
and $R(S) = g(A(S)) \sum_{i \in S} \beta_i - \kappa C(S)$, which
is~\eqref{eq:value} under~\eqref{eq:sd}.

\begin{lemmaA}[Batch errors]
\label{lem:batch}
Assumptions~\ref{ass:mag} and~\ref{ass:tied}. Let $\mathcal G_0$ be the
event that for every $i \in [N]$, with $j$ the batch of $i$: (G1)
$N_i = \sum_{t \in I_j} \xi_{i,t} \ge n p_i / 2$; (G2)
$|\hat q_i - q_i(\mathcal J_j)| \le \epsilon_0$; (G3)
$|\bar Y_i - \mathbb E Y_i| \le 2 \epsilon_0$, where
$\mathbb E Y_i = w_i q_i(\mathcal J_j) = \beta_i g(A(\mathcal J_j))$. If
$n_0 \ge 8 \ln(6NT) / p_{\min}$, then $\Pr(\mathcal G_0^c) \le 5 / (6T)$,
and on $\mathcal G_0$, for every $j$ and $i \in \mathcal J_j$,
\[
|\hat A_j - A(\mathcal J_j)| \le \frac{s \epsilon_0}{\lambda \gamma_s} =: \epsilon_A,
\qquad
|\hat a_i - a_i| \le \frac{\epsilon_0}{\lambda g_s} + \frac{L_{1/g}}{\lambda} \epsilon_A =: \epsilon_a,
\qquad
|\hat\beta_i - \beta_i| \le \frac{2 \epsilon_0}{g_s} + L_{1/g}\, \epsilon_A =: \epsilon_\beta .
\]
\end{lemmaA}

\begin{proof}
\emph{Probability.} The exploration schedule is fixed in advance, so for
$i \in \mathcal J_j$ the latent draws $(\xi_{i,t}, F_{i,t}, X_{i,t})_{t \in I_j}$
are i.i.d.\ across $t$ by Assumption~\ref{ass:obs}, and $X_{i,t}$ is
independent of $(\xi_{i,t}, F_{i,t})$ because it depends on the other
draws only through $\zeta_{i,t}$, which is independent of them; hence
$\mathbb E Y_i = p_i q_i(\mathcal J_j) r_i = w_i q_i(\mathcal J_j)$. By the
multiplicative Chernoff bound,
$\Pr(N_i < n p_i/2) \le e^{-n p_i/8} \le e^{-n_0 p_{\min}/8} \le 1/(6NT)$.
Conditional on the relevance draws, the compliance draws at the relevant
sessions are i.i.d.\ $\Bern(q_i(\mathcal J_j))$, so on (G1) Hoeffding's
inequality gives $\Pr(|\hat q_i - q_i(\mathcal J_j)| > \epsilon_0 \mid \xi)
\le 2 e^{-2 N_i \epsilon_0^2} \le 2 e^{-n p_i \epsilon_0^2}
\le 2 e^{-\ln(6NT)} = 1/(3NT)$. The $Y_{i,t}$ are i.i.d.\ in $[-1, 1]$, so
Hoeffding's inequality gives
$\Pr(|\bar Y_i - \mathbb E Y_i| > 2\epsilon_0) \le 2 e^{-2 n \epsilon_0^2}
\le 1/(3NT)$. A union bound over the $N$ items gives
$\Pr(\mathcal G_0^c) \le N (1/(6NT) + 2/(3NT)) = 5/(6T)$.

\emph{Load.} Write $A_j = A(\mathcal J_j) \le x_s$ and
$Q_j = Q(\mathcal J_j) = \lambda G(A_j)$, so $Q_j / \lambda \in [0, G(x_s)]$.
By (G2), $|\sum_{l} \hat q_l / \lambda - Q_j / \lambda| \le s \epsilon_0 / \lambda$,
and clipping to $[0, G(x_s)]$, which contains $Q_j/\lambda$, does not
increase the distance. On $[0, x_s]$, $G' \ge \gamma_s > 0$, so $G$ is
increasing there and $G^{-1}$ is $1/\gamma_s$-Lipschitz on
$[0, G(x_s)]$; hence $|\hat A_j - A_j| \le \epsilon_A$.

\emph{Items.} Since $\hat A_j, A_j \in [0, x_s]$, $g(\hat A_j) \ge g_s$ and
$|1/g(\hat A_j) - 1/g(A_j)| \le L_{1/g} \epsilon_A$. With
$q_i = q_i(\mathcal J_j) = \lambda a_i g(A_j) \le 1$,
\[
\frac{\hat q_i}{\lambda g(\hat A_j)} - a_i
= \frac{\hat q_i - q_i}{\lambda g(\hat A_j)}
+ \frac{q_i}{\lambda} \Big( \frac{1}{g(\hat A_j)} - \frac{1}{g(A_j)} \Big),
\]
of absolute value at most $\epsilon_a$; clipping to
$[0, a_{\max}] \ni a_i$ does not increase the error. Likewise, with
$|\mathbb E Y_i| \le 1$,
$\hat\beta_i - \beta_i = (\bar Y_i - \mathbb E Y_i)/g(\hat A_j)
+ \mathbb E Y_i\, (1/g(\hat A_j) - 1/g(A_j))$, of absolute value at most
$\epsilon_\beta$.
\end{proof}

\begin{lemmaA}[Plug-in stability]
\label{lem:plugin}
Assumptions~\ref{ass:mag} and~\ref{ass:tied}. On $\mathcal G_0$, every
$S \subseteq [N]$ with $|S| \le n_B$ satisfies
$|\hat R(S) - R(S)| \le |S|\, \varepsilon_{\mathrm{SD}}(n_0)$.
\end{lemmaA}

\begin{proof}
Both $A(S) = \sum_S a_i$ and $\hat A(S) = \sum_S \hat a_i$ lie in
$[0, x_B]$, and $|\hat A(S) - A(S)| \le |S| \epsilon_a$ by
Lemma~\ref{lem:batch}. Decompose
\[
\hat R(S) - R(S)
= g(\hat A(S)) \sum_{i \in S} (\hat\beta_i - \beta_i)
+ \big( g(\hat A(S)) - g(A(S)) \big) \sum_{i \in S} \beta_i .
\]
The first term is at most $|S| \epsilon_\beta$ in absolute value, as
$g \le 1$. For the second, $|\sum_S \beta_i| \le \lambda \sum_S |w_i| a_i
\le \lambda A(S)$, and by the definition of $K_g$,
$A(S)\, |g(\hat A(S)) - g(A(S))| \le K_g |\hat A(S) - A(S)|
\le K_g |S| \epsilon_a$. So
$|\hat R(S) - R(S)| \le |S| (\epsilon_\beta + \lambda K_g \epsilon_a)$, and
$\epsilon_\beta + \lambda K_g \epsilon_a = \varepsilon_{\mathrm{SD}}(n_0)$ by
the definitions of $\epsilon_a$, $\epsilon_\beta$ and $\epsilon_A$.
\end{proof}

The factor $A(S)$ in front of the difference of $g$ is what keeps the
commit error linear in $n_B$: a per-item bound on
$|\hat q_i(S) - q_i(S)|$ would carry $n_B$ once more through
$\hat A(S)$.

\begin{theoremA}[ETC-$m$ under an edit cap, full form]
\label{thm:etcfull}
Assumptions~\ref{ass:obs}, \ref{ass:mag} and~\ref{ass:tied}; $m \ge 2$,
$B \ge c_{\max}$, $c_i > 0$, $\kappa > 0$, $N, L, T \ge 1$. If
$n_0 \ge 8 \ln(6NT) / p_{\min}$, then
\[
\Reg_T \le \Delta_{\max} L \Big( J + \Big\lceil \frac{s + n_B}{m} \Big\rceil \Big)
+ \Lambda_T(n_0),
\qquad
\Lambda_T(n_0) = \Delta_{\max} J n_0 + 2 n_B T\, \varepsilon_{\mathrm{SD}}(n_0) + \Delta_{\max} .
\]
\end{theoremA}

\begin{proof}
\emph{Feasibility.} A batch costs at most $s c_{\max} \le B$, the sets
played during the move are subsets of the last batch or of $\hat S$, and
$C(\hat S) \le B' \le B$; so every played file is feasible and every
session's expected regret is at most $\Delta_{\max}$. A batch switch uses
at most $2s \le m$ edits and each epoch of the move at most $m$.

\emph{Step 1 (decomposition).} Exploration lasts
$J L \lceil n_0/L \rceil < J(n_0 + L)$ sessions and the move
$L \lceil (s + n_B)/m \rceil$ sessions, in all at most
$\Delta_{\max} [L(J + \lceil (s + n_B)/m \rceil) + J n_0]$ regret. Every
later session costs $R(S^\star) - R(\hat S)$.

\emph{Step 2 (commit error).} Some optimal file lies in $\mathcal S'$ and
every $S \in \mathcal S'$ has $|S| \le n_B$. On $\mathcal G_0$,
Lemma~\ref{lem:plugin} gives
$|\hat R(S) - R(S)| \le n_B \varepsilon_{\mathrm{SD}}(n_0)$ for all
$S \in \mathcal S'$, so
$R(S^\star) - R(\hat S) \le [R(S^\star) - \hat R(S^\star)]
+ [\hat R(S^\star) - \hat R(\hat S)] + [\hat R(\hat S) - R(\hat S)]
\le 2 n_B \varepsilon_{\mathrm{SD}}(n_0)$, the middle bracket being
$\le 0$ by the choice of $\hat S$.

\emph{Step 3 (expectation).} The commit phase lasts at most $T$ sessions,
so its regret is at most $2 n_B T \varepsilon_{\mathrm{SD}}(n_0)$ on
$\mathcal G_0$ and at most $T \Delta_{\max}$ off $\mathcal G_0$, whose
probability is at most $5/(6T) \le 1/T$ by Lemma~\ref{lem:batch}; in
expectation this adds at most $\Delta_{\max}$. Summing with Step~1 gives
the bound.
\end{proof}

\begin{corollaryA}[Tuned exploration]
\label{cor:tuned}
In the setting of Theorem~\ref{thm:etcfull}, let
$n_{\min} = \lceil 8 \ln(6NT) / p_{\min} \rceil$,
$C_1 = 2 n_B \big[ (2 + K_g)/g_s + (1 + K_g) L_{1/g}\, s / (\lambda \gamma_s) \big]
\sqrt{\ln(6NT) / p_{\min}}$, so that
$2 n_B \varepsilon_{\mathrm{SD}}(n_0) = C_1 n_0^{-1/2}$, and
$n_0 = \max\{n_{\min}, \lceil (T C_1 / (2 J \Delta_{\max}))^{2/3} \rceil\}$,
reading $y/0 = 0$. Then
\[
\Lambda_T(n_0) \le 1.9\, (J \Delta_{\max})^{1/3} (C_1 T)^{2/3}
+ J \Delta_{\max} (n_{\min} + 1) + \Delta_{\max} .
\]
\end{corollaryA}

\begin{proof}
If $\Delta_{\max} = 0$ then $\Reg_T \le T \Delta_{\max} = 0$ and there is
nothing to prove. Otherwise the first two terms of $\Lambda_T(n_0)$ are
$\mathfrak a n_0 + \mathfrak b n_0^{-1/2}$ with $\mathfrak a = J \Delta_{\max}$
and $\mathfrak b = C_1 T$; put $n^* = (\mathfrak b/(2\mathfrak a))^{2/3}$, so
$\mathfrak a n^* = 2^{-2/3} \mathfrak a^{1/3} \mathfrak b^{2/3}$ and
$\mathfrak b (n^*)^{-1/2} = 2^{1/3} \mathfrak a^{1/3} \mathfrak b^{2/3}$.
If $n_0 = \lceil n^* \rceil \ge n_{\min}$, then
$\mathfrak a n_0 + \mathfrak b n_0^{-1/2} \le \mathfrak a (n^* + 1)
+ \mathfrak b (n^*)^{-1/2} = 3 \cdot 2^{-2/3} \mathfrak a^{1/3} \mathfrak b^{2/3}
+ \mathfrak a \le 1.9\, \mathfrak a^{1/3}\mathfrak b^{2/3} + \mathfrak a$.
If $n_0 = n_{\min} > n^*$, then $\mathfrak a n_0 = \mathfrak a n_{\min}$ and
$\mathfrak b n_0^{-1/2} < \mathfrak b (n^*)^{-1/2} \le 1.9\, \mathfrak a^{1/3} \mathfrak b^{2/3}$.
In both cases the two terms are at most
$1.9\, \mathfrak a^{1/3} \mathfrak b^{2/3} + \mathfrak a (n_{\min} + 1)$.
\end{proof}

The tuning uses $\Delta_{\max}$, which the curator does not know. Tuning
with any known $\bar\Delta \ge \Delta_{\max}$ in its place, for instance
$\bar\Delta = n_B + \kappa B$ (as $|w_i| \le 1$ and $q_i \le 1$), gives
the same bound with $\bar\Delta$ in place of $\Delta_{\max}$ in the first
two terms, by the same proof with $\mathfrak a \le J \bar\Delta$.

\begin{proof}[Proof of Theorem~\ref{thm:etc}]
The hypotheses of Theorem~\ref{thm:etc} are those of
Theorem~\ref{thm:etcfull}, with $n_B$ the largest size of a file in
$\mathcal S'$, so the bound is Theorem~\ref{thm:etcfull}. Under the
tuning of Corollary~\ref{cor:tuned}, $C_1$ is
$O(\sqrt{\ln(NT)})$ and $n_{\min}$ is $O(\ln(NT))$ with the other
parameters fixed, so
$\Lambda_T = O(T^{2/3} (\ln NT)^{1/3})$. For the edit-cap term, if
$m \le 2B/c_{\max}$ then $s = \lfloor m/2 \rfloor \ge m/3$, so
$J \le 3N/m + 1$, and $n_B \le N$ gives
$\lceil (s + n_B)/m \rceil \le N/m + 2$; hence the term is at most
$\Delta_{\max} L (4N/m + 3)$.
\end{proof}

\begin{proof}[Proof of Theorem~\ref{thm:lower}]
\emph{Instance.} Fix $c \in (0, B]$, $p \in (0, 1]$, $\theta > 0$ and
$a > 0$ with $\theta g(a) < 1$ and $\kappa c < p \theta g(a)$; such
values exist since $c$ may be taken small. All items share $c$, $p$,
$\theta_i = \theta$ and $a_i = a$; harm $h_i = 0$; gap detection
$\delta$; under instance $j$ the gains are $r_j = 1$ and $r_i = 0$ for
$i \ne j$, observed without noise ($X_{i,t} \equiv r_i$). Then
$R(S) = p\, \theta\, g(|S| a)\, \mathbf 1\{j \in S\} - \kappa c |S|$. If
$j \in S$ this is at most $p \theta g(a) - \kappa c = R(\{j\})$, as $g$
is nonincreasing, and if $j \notin S$ it is $-\kappa c |S| \le 0 < R(\{j\})$.
So $S^\star_j = \{j\}$, and $R(S^\star_j) - R(S_t) \ge \Delta$ whenever
$j \notin S_t$; also $R(S) \ge -\kappa C(S) \ge -\kappa B$ for feasible
$S$, so $\Delta_{\max} \le \Delta + \kappa B$.

\emph{Reduction.} Let instance $0$ have $r_i = 0$ for all $i$. For a file
$S \not\ni j$, the observation law of a session with file $S$ is the same
under instance $j$ and instance $0$: in-file items have the same
compliance $\theta g(|S| a)$ and the same gain $0$ under both, and $j$
contributes only $G_{j,t} \sim \Bern(\delta p)$. This holds for every
joint law of the flags within a session, since the flag laws do not
depend on the gains. Let $\sigma_j = \min\{t : j \in S_t\}$ be the first
session that tries $j$. The law of a history
$(S_1, O_1, \dots, S_t, O_t)$ is the product of the policy kernels
$\mathsf P(S_s \mid \text{past})$ and the session laws
$\mathsf K(O_s \mid S_s)$; on the event $\{\sigma_j > t\}$ every factor
agrees under instances $j$ and $0$, so
$\Pr_j(\sigma_j > t) = \Pr_0(\sigma_j > t)$.

\emph{Counting.} From $S_0 = \emptyset$ with at most $m$ additions per
epoch, $|\bigcup_{s \le t} S_s| \le \sum_{s \le t} |S_s \setminus S_{s-1}|
\le m\, e(t)$ almost surely, so
$\sum_j \Pr_0(\sigma_j \le t) \le \mathbb E_0 |\bigcup_{s \le t} S_s|
\le m\, e(t)$, and
$\sum_j \Pr_j(j \notin S_t) \ge \sum_j \Pr_0(\sigma_j > t) \ge N - m\, e(t)$.

\emph{Summing.} $\Reg_T(j) \ge \Delta \sum_{t \le T} \Pr_j(j \notin S_t)$,
so, using $T \ge L n_m$ and $N - m e \ge 0$ for $e \le n_m$,
\[
\sum_j \Reg_T(j) \ge \Delta \sum_{e = 1}^{n_m} L (N - m e)
= \Delta L \Big(N n_m - \frac{m n_m (n_m + 1)}{2}\Big)
\ge \Delta L\, \frac{N (n_m - 1)}{2},
\]
because $m n_m \le N$. The maximum over $j$ is at least the average,
$(\Delta L/2)(n_m - 1)$. For $N \ge 4m$, $n_m - 1 > N/m - 2 \ge N/(2m)$.
\end{proof}

\section{EXPERIMENTAL DETAILS}
\label{app:experiments}

The plan of each experiment is paraphrased below, and every deviation
from the plan is listed.
Model calls go to hosted Claude models through their command-line
client; the context file is the whole system prompt, with no tools and
prompt caching off. Intervals are 95\% percentile intervals of a
task-cluster bootstrap with 2,000 replicates unless stated.

\subsection{Size sweep}
\label{app:smain}

\paragraph{Purpose.}
Proposition~\ref{prop:size} bounds the optimal file size independently of
the pool, and Corollary~\ref{cor:accretion} says the loss from loading
the whole pool grows with the pool. The size sweep grows a pool of real
context-file instructions from 200 to 3,200 items and measures value
against file size at each pool size, for two models on the same files.

\paragraph{Data.}
Candidates are root context-file list items from a public corpus of
coding-agent configuration files \citep{galster2026configs} (CC~BY~4.0
compilation; file contents keep their repository licenses, so only
permissive repositories are used and item text stays out of the released
data). Tasks are ManyIFEval prompts \citep{harada2025instructions}
(CC~BY~4.0); verifiers are the IFEval \citep{zhou2023ifeval} and IFBench
\citep{pyatkin2025ifbench} checkers (Apache-2.0). The subjects are Claude
Haiku~4.5 and Claude Sonnet~5.5.

\paragraph{Design.}
The pool holds 3,200 items drawn from the tier whose verbatim text cannot
apply to a reply, in one seeded order: one item per repository for the
first 1,901 items and a second item per repository after that. $N \in
\{200, 800, 3200\}$ takes prefixes. Every fifth item stays verbatim as a
zero-value candidate (640 items). The other 2,560 get a verifier on the
item's rarest word across the corpus (document frequency at most 1\%):
two thirds use IFEval \texttt{keywords:existence}, one third
\texttt{keywords:frequency} with a count of 2 or 3. Every fifth
verifiable item is genre-free; the rest carry one of six genre triggers
(``When the reply is an email, \dots''). Claude Sonnet~5.5 rewrites each
verifiable item's body, keeping the original wording; the trigger and
the verifier arguments are fixed by code, and a body that lacks the
keyword falls back to a template (3 items). Tasks are 60 ManyIFEval
prompts that name no genre; the genre is assigned round-robin (ten each)
and stated in the user turn with a 900--1,200 word range.

The file of size $b$ holds the first $b$ verifiable items of the $N$
pool in a second seeded order, zero-value items last. Files are
therefore nested in $b$, the file $b = N$ is the whole pool, and
$n^\star(N)$ is the optimum over these nested files. The grid is $N \in
\{200, 800, 3200\}$ and $b \in \{5, 10, 20, 40, 80, 200, 400, 800, 1600,
N\}$ with $b \le N$, 60 tasks, 2 replicates that differ in item order:
24 cells and 2,880 runs per model. Both models see the same files. Every
run is scored on every verifiable item of its pool that is relevant to
its genre, loaded or not.

A keyword also appears unprompted. Its chance rate is the share of
relevant unloaded verifiers that pass, per model, pool and task, over
the files of at most 40 items. It is $0.045$, $0.024$ and $0.016$ at $N
= 200$, $800$ and $3200$ for Haiku and within $0.003$ of that for
Sonnet. Per cell, $\Rz(S_b)$ is the mean number of followed relevant
loaded items minus the chance rate times their number. A singleton
screen loads an item alone on 6 tasks; $q_i(\{i\})$ is the mean over
them. It covers the 452 items loaded in some file of size at most $200$,
once per model (2,712 runs each). The load $\ell(S_b)$ is the sum of the
subject's singleton rates over the relevant loaded items; an unscreened
item takes the mean rate of the screened items of its kind and verifier
family, which is unbiased because files are seeded prefixes. The mean
singleton rate is $0.58$ for Haiku and $0.80$ for Sonnet. $C(S_b)$ is
the file's tokens. The marginal relevant compliance between consecutive
sizes is $\mathrm{MC}(b) = [\Rz(S_{b'}) - \Rz(S_b)] / [\ell(S_{b'}) - \ell(S_b)]$. The
value at price $\rho_{\mathrm{exp}}$ is $R_{\rho_{\mathrm{exp}}}(S_b) = \Rz(S_b) - (\rho_{\mathrm{exp}} /
c_{\mathrm{med}}) C(S_b)$ with $c_{\mathrm{med}} = 39$ tokens, the median
verifiable item, so $\rho_{\mathrm{exp}}$ is the normalized cost per
median-length item; it is distinct from the theoretical break-even
compliance $\rho$. From one task-cluster bootstrap we take
$n^\star(N) = \arg\max_b R_{\rho_{\mathrm{exp}}}(S_b)$ and the accretion loss
$R_{\rho_{\mathrm{exp}}}(S_{n^\star}) - R_{\rho_{\mathrm{exp}}}(S_N)$.

\paragraph{Details.}
At $b = 3200$ the chance hits are about 14 of the 853 relevant loaded
items, most of the raw count of followed items (14.3 for Haiku, 18.1 for
Sonnet); Tables~\ref{tab:smaincellshaiku}
and~\ref{tab:smaincellssonnet} give the raw counts next to the corrected
ones. The relevant load at $b = N$ is about $0.27 N$ items (53, 213,
853). The singleton rate of a genre-free item is pooled over genres.
ManyIFEval has too few prompts of each genre, so the genre is assigned
and stated in the user turn. The task's own instruction list is not
used. The median item costs 39 tokens after rewriting, against the
corpus median of 15. No run was truncated or empty.

\paragraph{Results.}
Tables~\ref{tab:smaincellshaiku} and~\ref{tab:smaincellssonnet} give the
cells and Table~\ref{tab:smainopt} the measured optima. At the top
step, $b = 1600 \to 3200$, $\mathrm{MC}$ is $-0.007$ $[-0.013, -0.002]$ for Haiku
and $-0.037$ $[-0.052, -0.023]$ for Sonnet. At $N =
3200$ gross value peaks at $b = 200$, at $5.0$ followed instructions for
Haiku and $23.8$ for Sonnet, and falls to $0.6$ $[0.0, 1.3]$ and $4.0$
$[0.5, 9.5]$ for the whole pool (Figure~\ref{fig:smain}). Without the
chance correction the top-step $\mathrm{MC}$ is $0.029$ $[0.022, 0.035]$ for Haiku
and $-0.013$ $[-0.028, 0.001]$ for Sonnet. Over the three pools the
optimum at $\rho_{\mathrm{exp}} = 0.1$ is $10$, $5$, $10$ for Haiku and $40$, $20$,
$40$ for Sonnet: it does not grow with the pool. From $N = 800$ to $N =
3200$ the accretion loss grows from $76.2$ to $316.7$ for Haiku and from
$75.9$ to $320.4$ for Sonnet, a factor $4.2$ for each. Sonnet follows about five times as many instructions as Haiku
before capacity binds, and its optimal file is four times as large.

The curves are not smooth, because unranked files differ in the
difficulty of their items. For Sonnet at $N = 800$ gross value is $28.4$
at $b = 400$ and $3.7$ for the whole pool, a step of $-0.46$ $[-0.56,
-0.36]$. Loaded items of another genre are followed at $0.02$--$0.05$.
Replies fall inside the word range in $76\%$ of runs for both models.
Figure~\ref{fig:smainvalue} shows the value curves at $\rho_{\mathrm{exp}} = 0.1$ and
Figure~\ref{fig:smainmarg} the marginal compliance against the load.
Figures~\ref{fig:smainrhohaiku} and~\ref{fig:smainrhosonnet} show the
value curves at the other prices: the optimum moves left as $\rho_{\mathrm{exp}}$
rises, and it sits at $b = N$ only for Sonnet at $N = 200$ and $\rho_{\mathrm{exp}} =
0.05$.

\paragraph{Monotone dilution, item by item.}
Assumption~\ref{ass:mono} is a statement about each instruction: its
compliance does not rise when the file grows. Files are nested in $b$
within a pool, so every relevant item of $S_b$ is in every larger file,
and the per-item labels test the assumption directly. For each pool we
follow the relevant items of the $b = 40$ file, 12 to 16 per task,
through the larger files (Figures~\ref{fig:smainmonohaiku}
and~\ref{fig:smainmonosonnet}). At $N = 3200$ their compliance falls
from $0.27$ to $0.05$ for Haiku and from $0.91$ to $0.05$ for Sonnet. At
$N = 800$ it falls from $0.29$ to $0.04$ and from $0.47$ to $0.07$, and
at $N = 200$ from $0.25$ to $0.11$ and from $0.55$ to $0.31$. These are
raw rates, with the chance rate of $0.02$--$0.05$ as their floor.

Order matters inside large files, which the set-level model does not
represent: at $b = 800$ of the largest pool the first fifth of the file
is followed about twice as often as the rest by both models ($0.05$
against $0.02$ for Haiku, $0.14$ against $0.06$ for Sonnet).

Against their singleton rates, the screened items are followed less
often in a file than alone in 23 of the 24 cells for Haiku and 22 for
Sonnet, by $0.07$ to $0.66$. In the other three cells, all files of at
most 40 items, they are followed as often or more often, by at most
$0.14$.

\begin{figure}[h]
\centering
% Numbers: data/v2/smain-sonnet/processed/cells.csv and optimum.csv (rho = 0.1).
\includegraphics[width=\textwidth]{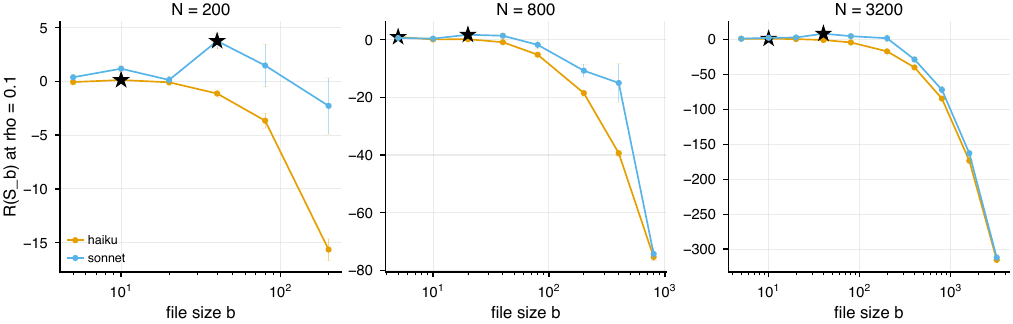}
\caption{Size sweep: value $R_{\rho_{\mathrm{exp}}}(S_b)$ at price $\rho_{\mathrm{exp}} = 0.1$ against
  file size $b$ for both models, one panel per pool size $N$. Points are
  means over the 60 tasks; bars are $\pm 1.96$ Newey--West standard
  errors over the task sequence. The stars mark the measured optima
  $n^\star(N)$.}
\label{fig:smainvalue}
\end{figure}

\begin{figure}[h]
\centering
\includegraphics{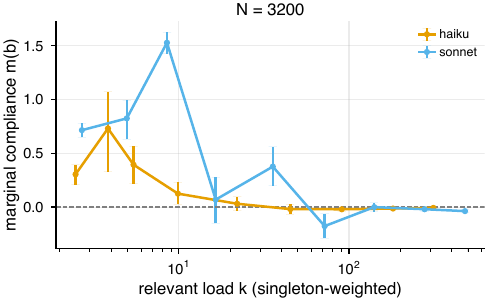}
\caption{Size sweep: marginal relevant compliance $\mathrm{MC}(b)$ against relevant
  load at $N = 3200$ for both models, with 95\% task-cluster bootstrap
  bars.}
\label{fig:smainmarg}
\end{figure}

\begin{table}[h]
\caption{Size-sweep cells for Claude Haiku~4.5: gross value $\Rz(S_b)$
  (followed relevant instructions per task, net of chance), the raw
  count, the singleton-weighted relevant load $\ell(S_b)$, and the marginal
  compliance $\mathrm{MC}$ of the step ending at $b$ with its 95\% task-cluster
  bootstrap interval.}
\label{tab:smaincellshaiku}
\centering
\small
% Numbers: data/v2/smain-sonnet/processed/cells.csv and steps.csv (model = haiku)
\begin{tabular}{lrrrrl}
$N$ & $b$ & $\Rz(S_b)$ & raw & $\ell(S_b)$ & $\mathrm{MC}$ (step to $b$) \\
\hline
200 & 5 & 0.53 & 0.60 & 1.33 &  \\
 & 10 & 1.19 & 1.34 & 1.58 & 2.67 [1.91, 3.88] \\
 & 20 & 2.05 & 2.35 & 3.39 & 0.48 [0.37, 0.57] \\
 & 40 & 3.18 & 3.89 & 8.69 & 0.21 [0.13, 0.28] \\
 & 80 & 5.11 & 6.49 & 14.00 & 0.36 [0.23, 0.50] \\
 & 200 & 3.84 & 6.25 & 25.06 & $-0.11$ [$-0.22$, $-0.01$] \\
800 & 5 & 1.37 & 1.43 & 2.00 &  \\
 & 10 & 1.01 & 1.10 & 2.53 & $-0.68$ [$-1.49$, $-0.14$] \\
 & 20 & 2.16 & 2.30 & 3.53 & 1.16 [0.92, 1.41] \\
 & 40 & 3.33 & 3.62 & 7.03 & 0.33 [0.21, 0.44] \\
 & 80 & 3.65 & 4.18 & 13.53 & 0.05 [$-0.03$, 0.14] \\
 & 200 & 3.41 & 4.92 & 32.50 & $-0.01$ [$-0.06$, 0.05] \\
 & 400 & 4.05 & 7.19 & 62.93 & 0.02 [$-0.01$, 0.05] \\
 & 800 & 2.65 & 7.69 & 98.36 & $-0.04$ [$-0.07$, $-0.01$] \\
3200 & 5 & 1.55 & 1.57 & 1.56 &  \\
 & 10 & 2.11 & 2.17 & 3.39 & 0.30 [0.21, 0.40] \\
 & 20 & 2.76 & 2.88 & 4.28 & 0.73 [0.33, 1.07] \\
 & 40 & 3.65 & 3.88 & 6.56 & 0.39 [0.22, 0.56] \\
 & 80 & 4.49 & 4.95 & 13.19 & 0.13 [0.03, 0.23] \\
 & 200 & 5.04 & 6.12 & 30.86 & 0.03 [$-0.04$, 0.10] \\
 & 400 & 4.50 & 6.63 & 59.47 & $-0.02$ [$-0.06$, 0.03] \\
 & 800 & 3.24 & 7.51 & 121.98 & $-0.02$ [$-0.04$, 0.00] \\
 & 1600 & 1.67 & 10.16 & 240.79 & $-0.01$ [$-0.02$, 0.00] \\
 & 3200 & 0.63 & 14.33 & 385.89 & $-0.01$ [$-0.01$, 0.00] \\
\end{tabular}
\end{table}

\begin{table}[h]
\caption{Size-sweep cells for Claude Sonnet~5.5, as in
  Table~\ref{tab:smaincellshaiku}.}
\label{tab:smaincellssonnet}
\centering
\small
% Numbers: data/v2/smain-sonnet/processed/cells.csv and steps.csv (model = sonnet)
\begin{tabular}{lrrrrl}
$N$ & $b$ & $\Rz(S_b)$ & raw & $\ell(S_b)$ & $\mathrm{MC}$ (step to $b$) \\
\hline
200 & 5 & 0.97 & 1.04 & 1.56 &  \\
 & 10 & 2.25 & 2.39 & 3.19 & 0.78 [0.67, 0.88] \\
 & 20 & 2.27 & 2.56 & 5.28 & 0.01 [$-0.09$, 0.12] \\
 & 40 & 8.07 & 8.74 & 11.28 & 0.97 [0.83, 1.11] \\
 & 80 & 10.24 & 11.55 & 19.00 & 0.28 [0.05, 0.51] \\
 & 200 & 17.21 & 19.45 & 35.22 & 0.43 [0.25, 0.60] \\
800 & 5 & 1.02 & 1.08 & 2.28 &  \\
 & 10 & 1.29 & 1.39 & 2.78 & 0.53 [$-0.09$, 1.00] \\
 & 20 & 3.71 & 3.87 & 4.42 & 1.48 [1.31, 1.66] \\
 & 40 & 5.58 & 5.91 & 10.61 & 0.30 [0.13, 0.47] \\
 & 80 & 7.07 & 7.66 & 18.42 & 0.19 [0.02, 0.38] \\
 & 200 & 11.18 & 12.85 & 46.44 & 0.15 [0.07, 0.23] \\
 & 400 & 28.40 & 31.86 & 93.65 & 0.36 [0.26, 0.48] \\
 & 800 & 3.72 & 9.27 & 147.59 & $-0.46$ [$-0.56$, $-0.36$] \\
3200 & 5 & 1.60 & 1.62 & 1.50 &  \\
 & 10 & 3.31 & 3.38 & 3.89 & 0.71 [0.65, 0.78] \\
 & 20 & 5.05 & 5.17 & 6.00 & 0.82 [0.63, 0.99] \\
 & 40 & 12.72 & 12.95 & 11.03 & 1.53 [1.42, 1.63] \\
 & 80 & 13.48 & 13.96 & 21.92 & 0.07 [$-0.15$, 0.28] \\
 & 200 & 23.82 & 24.94 & 49.42 & 0.38 [0.20, 0.56] \\
 & 400 & 15.88 & 18.07 & 94.43 & $-0.18$ [$-0.29$, $-0.06$] \\
 & 800 & 15.90 & 20.28 & 186.93 & 0.00 [$-0.04$, 0.04] \\
 & 1600 & 12.32 & 21.05 & 369.21 & $-0.02$ [$-0.04$, 0.00] \\
 & 3200 & 4.01 & 18.09 & 591.98 & $-0.04$ [$-0.05$, $-0.02$] \\
\end{tabular}
\end{table}

\begin{table}[h]
\caption{Size sweep: measured optimum $n^\star(N)$ and accretion loss
  $R_{\rho_{\mathrm{exp}}}(S_{n^\star}) - R_{\rho_{\mathrm{exp}}}(S_N)$ with 95\% task-cluster bootstrap
  interval, per model, pool size and price $\rho_{\mathrm{exp}}$.}
\label{tab:smainopt}
\centering
\footnotesize
\setlength{\tabcolsep}{4pt}
% Numbers: data/v2/smain-sonnet/processed/optimum.csv
\resizebox{\linewidth}{!}{%
\begin{tabular}{llcccc}
 & $N$ & $\rho_{\mathrm{exp}} = 0.05$ & $\rho_{\mathrm{exp}} = 0.1$ & $\rho_{\mathrm{exp}} = 0.2$ & $\rho_{\mathrm{exp}} = 0.4$ \\
\hline
Haiku & 200 & 40; 6.9 [5.9, 8.0] & 10; 15.8 [14.6, 16.8] & 5; 34.5 [33.3, 35.5] & 5; 72.3 [71.1, 73.3] \\
 & 800 & 40; 37.5 [36.4, 38.7] & 5; 76.2 [74.9, 77.3] & 5; 153.7 [152.4, 154.8] & 5; 308.7 [307.4, 309.8] \\
 & 3200 & 20; 159.1 [158.4, 159.9] & 10; 316.7 [316.0, 317.4] & 5; 632.2 [631.5, 632.9] & 5; 1263.5 [1262.8, 1264.1] \\
Sonnet & 200 & 200; 0.0 [0.0, 1.3] & 40; 6.0 [3.7, 8.3] & 10; 21.9 [19.3, 24.5] & 5; 59.3 [56.7, 62.0] \\
 & 800 & 400; 41.9 [37.9, 47.6] & 20; 75.9 [73.1, 78.4] & 5; 152.3 [149.5, 154.6] & 5; 307.3 [304.4, 309.6] \\
 & 3200 & 200; 166.9 [161.3, 173.0] & 40; 320.4 [315.1, 324.2] & 40; 632.1 [626.7, 635.9] & 5; 1260.2 [1254.7, 1263.7] \\
\end{tabular}}
\end{table}

\begin{figure}[h]
\centering
\includegraphics[width=\linewidth]{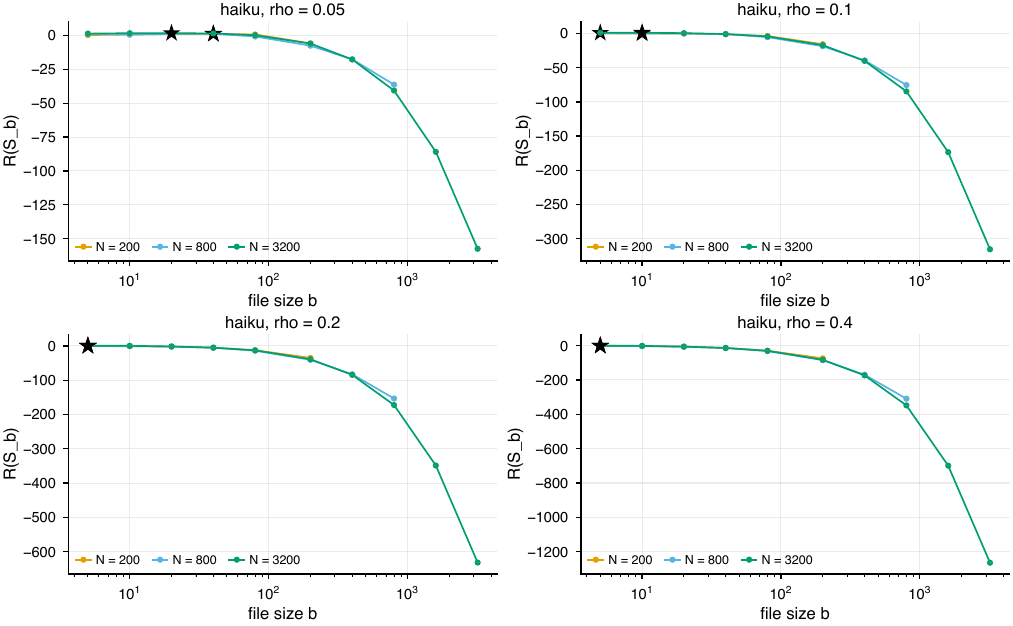}
\caption{Size sweep, Claude Haiku~4.5: value $R_{\rho_{\mathrm{exp}}}(S_b)$ against file
  size at each price $\rho_{\mathrm{exp}}$, one series per pool size $N$, with
  Newey--West error bars and stars at $n^\star(N)$ as in
  Figure~\ref{fig:smainvalue}.}
\label{fig:smainrhohaiku}
\end{figure}

\begin{figure}[h]
\centering
\includegraphics[width=\linewidth]{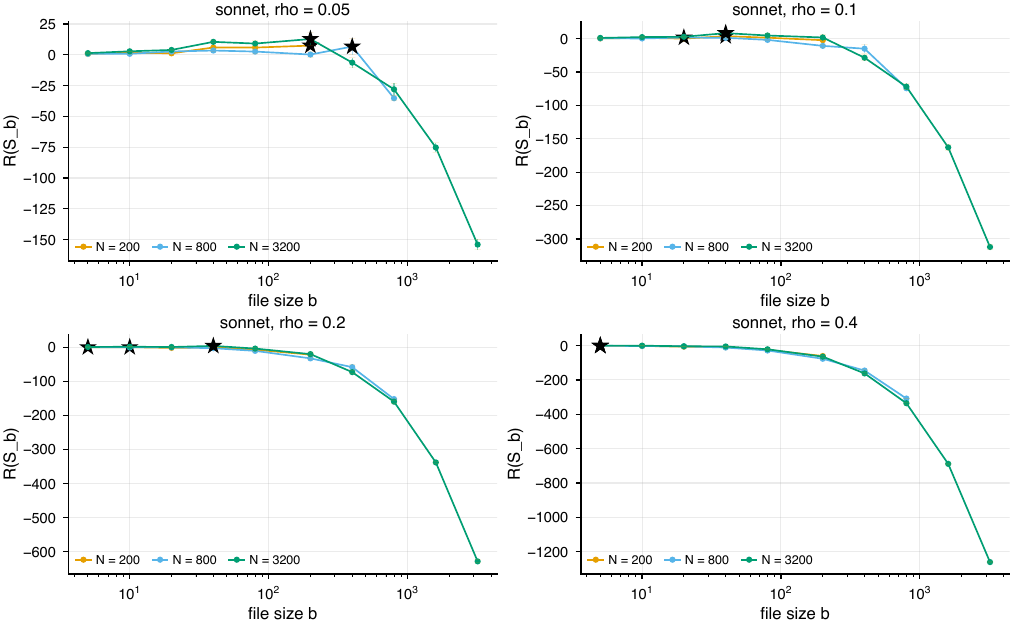}
\caption{Size sweep, Claude Sonnet~5.5: value $R_{\rho_{\mathrm{exp}}}(S_b)$ against file
  size at each price $\rho_{\mathrm{exp}}$, as in Figure~\ref{fig:smainrhohaiku}.}
\label{fig:smainrhosonnet}
\end{figure}

\begin{figure}[h]
\centering
% Numbers: data/v2/smain-sonnet/processed/mono_base.csv (run.py mono).
\includegraphics[width=\textwidth]{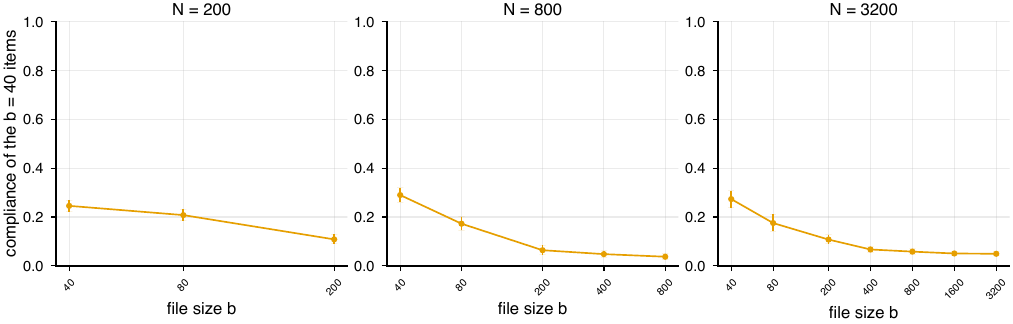}
\caption{Monotone dilution, item by item, Claude Haiku~4.5: compliance of
  the relevant items of the $b = 40$ file in that file and in every
  larger file of the same pool, with 95\% task-cluster bootstrap bars.}
\label{fig:smainmonohaiku}
\end{figure}

\begin{figure}[h]
\centering
\includegraphics[width=\textwidth]{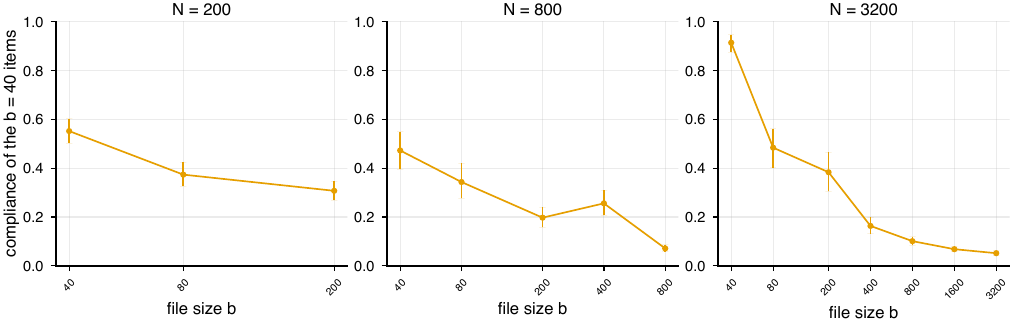}
\caption{Monotone dilution, item by item, Claude Sonnet~5.5, as in
  Figure~\ref{fig:smainmonohaiku}.}
\label{fig:smainmonosonnet}
\end{figure}

\clearpage
\subsection{Dilution}
\label{app:scal}

\paragraph{Purpose.}
Assumption~\ref{ass:mono} says that compliance with one instruction
does not increase as the file grows. The dilution experiment measures the compliance
curve $q(n)$ of 24 model configurations on a published
instruction-density benchmark, with no model calls of our own.

\paragraph{Data.}
The IFScale generations \citep{jaroslawicz2025ifscale}: 5,993 reports from
24 model configurations with $n \in \{10, \dots, 500\}$ keyword
instructions and 5 seeds. The repository has no license and the paper is
CC~BY-NC-SA~4.0, so the raw generations stay local and only aggregates
are reported.

\paragraph{Design.}
Every generation is re-scored with our labels; 52 generations that paste
the keyword list are dropped. Per configuration we fit $M(n) = n\, q(n)$,
the number of followed instructions, and report $M(500)/500$, the
compliance at $n = 500$. The fitted capacity $K$ of the corrected curve
fit sets the class: weak $K < 100$, mid $100$--$350$, strong $K > 350$.

\paragraph{Preregistered call.}
Descriptive only: $M(500)/500$ is expected between $0.07$ and $0.69$ in
all 24 configurations, with the late/early slope ratio of $M(n)$ below
$0.1$ in the four weak configurations and the hyperbola $g(x) = 1/(1+x)$
the AIC-best form only in the weak class.

\paragraph{Deviations.}
The fitted $K$ places 6 configurations in the weak class, 5 in the mid
class and 13 in the strong class; the plan projected 4, 8 and 12. The
slope ratio is below $0.1$ in 2 of the 6 weak configurations, and the
hyperbola is AIC-best in 4 of the 6 weak configurations and in none of
the others. One configuration, \texttt{llama-4-scout}, falls just
below the planned $0.07$ floor at $0.067$.

\paragraph{Results.}
Compliance at $n = 10$ is $0.92$--$1.00$ in every configuration.
$M(500)/500$ lies in $[0.067, 0.689]$ and is below $1$ in 24 of 24
configurations (Table~\ref{tab:scal}): no configuration follows its
instructions in proportion to their number at $n = 500$, and the
strongest configuration follows fewer than $0.7$ of them. The curve
$q(n)$ is nonincreasing for every class (Figure~\ref{fig:scal}). The
hyperbola of the logit case fits only the weak configurations; the
monotone dilution Assumption~\ref{ass:mono} needs no particular curve and holds
throughout.

\begin{table}[h]
\caption{Dilution: fitted capacity $K$, class, AIC-best curve form and
  compliance $M(500)/500$ at $n = 500$, per configuration.}
\label{tab:scal}
\centering
\small
% Numbers: data/v2/scal/processed/scal_summary.csv
\begin{tabular}{lrllc}
configuration & $K$ & class & best form & $M(500)/500$ \\
\hline
llama-4-scout & 34 & weak & power & 0.067 \\
claude-3.5-haiku & 52 & weak & hyperbola & 0.085 \\
gpt-4o-mini & 61 & weak & hyperbola & 0.104 \\
gpt-4o & 81 & weak & hyperbola & 0.154 \\
llama-4-maverick & 87 & weak & power & 0.096 \\
gpt-4.1-nano & 99 & weak & hyperbola & 0.162 \\
qwen3-235b-a22b & 161 & mid & exponential & 0.209 \\
gpt-4.1-mini & 202 & mid & exponential & 0.272 \\
deepseek-r1-0528 & 229 & mid & exponential & 0.309 \\
gemini-2.5-flash-preview & 281 & mid & exponential & 0.342 \\
grok-3-mini-beta & 329 & mid & exponential & 0.364 \\
gpt-4.5-preview & 376 & strong & exponential & 0.430 \\
claude-opus-4 & 428 & strong & linear & 0.446 \\
claude-sonnet-4 & 504 & strong & linear & 0.429 \\
gpt-4.1 & 526 & strong & linear & 0.489 \\
claude-opus-4 (high reasoning) & 602 & strong & exponential & 0.521 \\
o4-mini & 611 & strong & logistic & 0.344 \\
claude-3.7-sonnet & 613 & strong & linear & 0.527 \\
claude-sonnet-4 (high reasoning) & 648 & strong & logistic & 0.399 \\
o3 & 845 & strong & logistic & 0.516 \\
grok-3-beta & 1012 & strong & linear & 0.619 \\
gemini-2.5-pro-preview & 1274 & strong & linear & 0.666 \\
gemini-2.5-pro-preview (high reasoning) & 1333 & strong & linear & 0.689 \\
o3 (high reasoning) & 1607 & strong & logistic & 0.628 \\
\end{tabular}
\end{table}

\begin{figure}[h]
\centering
\includegraphics[width=\linewidth]{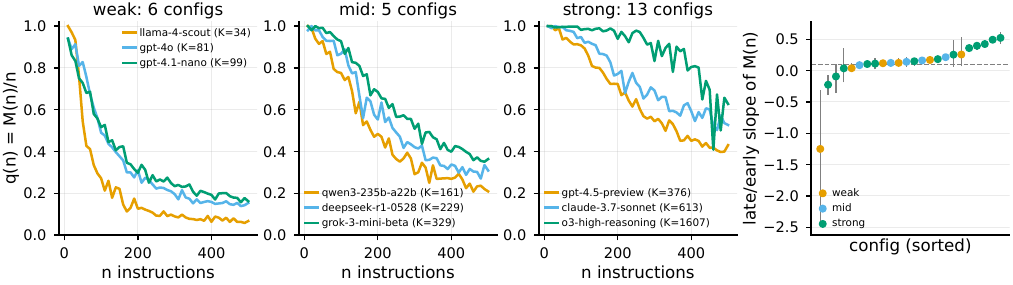}
\caption{Dilution: $q(n) = M(n)/n$ against $n$ for the lowest-, median- and
  highest-$K$ configuration of each class (three left panels), and the
  24 late/early slope ratios sorted, with 95\% seed-bootstrap bars and
  the $0.1$ reference line (right).}
\label{fig:scal}
\end{figure}

\subsection{Regret}
\label{app:l3}

\paragraph{Purpose.}
Theorem~\ref{thm:etc} bounds the edit-cap term of ETC-$m$ from above
and Theorem~\ref{thm:lower} bounds every policy from below, both of order
$(N/m) L$. The regret experiment runs ETC-$m$ on the instances of
Theorem~\ref{thm:lower} and checks both bounds and the order. It is a
verification of the theorems, not an empirical test; no dataset has an
edit cap. No data and no model calls.

\paragraph{Design.}
The instances are those of Theorem~\ref{thm:lower}: $N$ identical items,
$a = 0.02$, $\theta g(a) = 0.9$, $p = 0.3$, $c = 1$, $\kappa = 0.004$,
$B = 10$, $r_1 = 1$, noiseless gains, with $g(x) = 1/(1+x)$ of
Appendix~\ref{app:mnl}; $S^\star = \{1\}$ and every cell is exact. Grid:
$N \in \{16, 32, 64, 128, 256\}$, $m \in \{2, 4, 8, 16\}$, $L \in \{10,
20, 40\}$, 60 cells. ETC-$m$ is Algorithm~\ref{alg:etc} with the
estimator of Appendix~\ref{app:regret}, under tied loads $\theta_i =
\lambda a_i$. Regret is exact wherever the schedule fixes the file, so
the only random element is the committed file. The cap term is the
regret of the first epoch of every exploration batch plus the
transition; the learning term is the later exploration epochs plus the
commit phase. $n_0 = \lceil 8 \ln(6NT) / p_{\min} \rceil$ and $T = L(2J
+ \lceil (s + n_B)/m \rceil) + 4 n_0 J$ are solved as a fixed point.
The lower bound is $(\Delta L / 2)(n_m - 1)$ with $n_m = \lfloor N/m
\rfloor$ and $\Delta = R(S^\star)$.

\paragraph{Preregistered call.}
The ratio of the cap term to the lower bound lies in $[1, 4
(\Delta_{\max}/\Delta)(1 + m (1 + n_B/s)/(2N))]$ in every cell with $N
\ge 4m$, and the log-log slope of the cap term on $N/m$ is $1.0 \pm
0.1$ at each $L$. A ratio below $1$ refutes the bound or the simulator;
a slope outside the band refutes the order $N/m$.

\paragraph{Deviations.}
None. Gains are observed without noise and $\eta = 1$; the pool has no
spurious candidates, so $\varepsilon$ is unused.

\paragraph{Results.}
Both calls pass. Over the 51 cells with $N \ge 4m$ the ratio to the
lower bound is $4.08$--$5.38$, inside the interval in every cell
(Table~\ref{tab:l3a}); the cap term is under the upper bound in all 60
cells, and ETC-$m$ committed to $S^\star$ in every cell. The log-log
slope of the cap term on $N/m$ is $1.058$ $[1.029, 1.088]$ with $R^2 =
0.997$ at each $L$, over a span of a factor $192$ in the cap term
(Figure~\ref{fig:l3}). The ratio does not depend on $L$ and falls toward
$4$ as $N$ grows, the constant of Theorem~\ref{thm:etc}.

The three parameters move the cap term as the bounds predict. More
candidates increase the cost, because the policy must test more
instructions. A larger edit cap decreases the cost. Longer epochs
increase the time that the policy must keep each test batch.

\begin{table}[h]
\caption{Regret experiment: cap-term regret of ETC-$m$ over the lower
  bound $(\Delta L/2)(n_m - 1)$, with the plan's upper limit
  $4 (\Delta_{\max}/\Delta)(1 + m(1 + n_B/s)/(2N))$ in parentheses. The
  ratio does not depend on $L$. Cells with $N < 4m$ are outside the call.}
\label{tab:l3a}
\centering
\small
% Numbers: data/v2/l3/processed/l3_familyA.csv (ratio, ratio_max), identical across L
\begin{tabular}{lcccc}
$N$ & $m = 2$ & $m = 4$ & $m = 8$ & $m = 16$ \\
\hline
16 & 4.35 (7.77) & 4.83 (8.05) & $N < 4m$ & $N < 4m$ \\
32 & 4.20 (6.18) & 4.42 (6.33) & 5.02 (6.61) & $N < 4m$ \\
64 & 4.13 (5.39) & 4.26 (5.46) & 4.57 (5.61) & 5.38 (5.90) \\
128 & 4.09 (5.00) & 4.19 (5.03) & 4.40 (5.10) & 4.87 (5.25) \\
256 & 4.08 (4.80) & 4.15 (4.82) & 4.32 (4.85) & 4.66 (4.93) \\
\end{tabular}
\end{table}

\end{document}

%% file: experiments.tex
\section{EXPERIMENTS}
\label{sec:experiments}

We use three experiments to study instruction files.
The first measures how file size affects compliance and value.
The second measures compliance in responses from a published benchmark.
The third uses simulations to measure the effect of limits on file edits.
Appendix~\ref{app:experiments} gives the detailed designs, results,
and analysis plans.

\paragraph{Size sweep.}
The first experiment tests whether a larger candidate pool requires a
larger file.
The pool holds 3,200 items from real repository context files.
Claude Sonnet~5.5 rewrites 2,560 of them as instructions about the
words in a response, so that an automatic checker can score each
instruction.
Most of these instructions apply only to one response format, such as
an email.
The other 640 items stay unchanged and have zero value in our score.
The tasks are 60 ManyIFEval writing tasks
\citep{harada2025instructions}, ten in each of six response formats.
The file is the system prompt, and the task is the user prompt.

We vary the pool size $N \in \{200,800,3200\}$ and the file size
\[
b \in \{5,10,20,40,80,200,400,800,1600,N\}, \qquad b \le N.
\]
The file $S_b$ holds the $b$ instructions that the model receives.
Pools and files are nested: each larger pool contains the smaller
pools, and each larger file contains the smaller files.
Claude Haiku~4.5 and Claude Sonnet~5.5 receive the same files and
tasks, for 2,880 runs each.

We report three quantities.
The gross value $\Rz(S_b)$ is the mean number of relevant instructions
that the model follows, minus the number of keyword matches expected
by chance.
The relevant load $\ell(S_b)$ adds, over the relevant instructions in the
file, the compliance of each instruction when it is loaded alone
(Appendix~\ref{app:smain}).
The value subtracts a token penalty from the gross value,
\[
R_{\rho_{\mathrm{exp}}}(S_b) = \Rz(S_b) - \frac{\rho_{\mathrm{exp}}}{c_{\mathrm{med}}}\,C(S_b),
\]
where $C(S_b)$ is the token count of the file and
$c_{\mathrm{med}}\approx39$ tokens is the median instruction length.
Here $\rho_{\mathrm{exp}}$ is a normalized cost per median-length item,
distinct from the theoretical break-even compliance $\rho$.
At $\rho_{\mathrm{exp}}=0.1$, one median-length instruction costs $0.1$ units of
gross value.
The measured optimum $n^\star(N)$ is the size of the highest-value
file among the files we test.

\begin{figure*}[t]
\centering
\includegraphics[width=\textwidth]{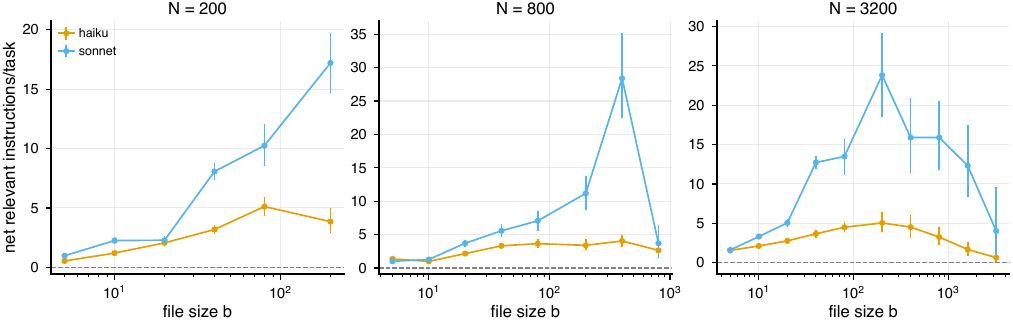}
\caption{
Gross value against file size for Claude Haiku~4.5 and
Claude Sonnet~5.5.
Gross value counts relevant instructions that the model follows,
after adjustment for chance keyword matches.
Each panel uses a different candidate pool.
Points show means across 60 tasks.
Bars show 95\% confidence intervals from a bootstrap that resamples
tasks and keeps each task's runs together.
The last point in each panel loads the entire pool.
}
\label{fig:smain}
\end{figure*}

\paragraph{Size sweep: results.}
Large files reduce gross value in the largest pool
(Figure~\ref{fig:smain}).
At $N=3200$, both models reach their highest gross value at $b=200$,
where Haiku follows $5.0$ relevant instructions and Sonnet follows
$23.8$.
With the entire pool loaded, these values decrease to $0.6$ and $4.0$.
The marginal relevant compliance, the change in gross value per unit
of additional relevant load, is negative at the last step.
From $b=1600$ to $b=3200$, it is $-0.007$ $[-0.013,-0.002]$ for Haiku
and $-0.037$ $[-0.052,-0.023]$ for Sonnet (95\% intervals).

The best tested file stays small as the pool grows.
At $\rho_{\mathrm{exp}}=0.1$ and $N=200$, $800$, and $3200$, the measured optimum is
$10$, $5$, and $10$ for Haiku and $40$, $20$, and $40$ for Sonnet.
This result is consistent with the size bound of
Proposition~\ref{prop:size}, although it covers only the tested files.
The accretion loss, $R_{\rho_{\mathrm{exp}}}(S_{n^\star})-R_{\rho_{\mathrm{exp}}}(S_N)$, compares the
best tested file with the file that loads the entire pool.
From $N=200$ to $N=3200$, it increases from $15.8$ to $316.7$ for
Haiku and from $6.0$ to $320.4$ for Sonnet.
The loss includes both the compliance effect and the token penalty,
as in Corollary~\ref{cor:accretion}.

Dilution also appears instruction by instruction.
At $N=3200$, the relevant instructions of the $b=40$ file have a mean
compliance of $0.27$ for Haiku and $0.91$ for Sonnet in that file.
With the entire pool loaded, both rates decrease to $0.05$, as
Assumption~\ref{ass:mono} predicts at high load.
The score measures instruction compliance and not the overall quality
of the response.

\paragraph{Dilution on a published benchmark.}
We re-score 5,993 existing IFScale responses
\citep{jaroslawicz2025ifscale} from 24 model configurations.
Compliance is $0.92$--$1.00$ at ten keyword requirements and
$0.067$--$0.689$ at 500 requirements.
Appendix~\ref{app:scal} gives the fitted curves and the results
for each configuration.

\paragraph{Regret under an edit cap.}
The third experiment runs ETC-$m$ (Algorithm~\ref{alg:etc}) on the
instances of Theorem~\ref{thm:lower}, where the optimal file is
$\{j\}$ and gains have no noise.
We vary $N \in \{16,32,64,128,256\}$, $m \in \{2,4,8,16\}$, and
$L \in \{10,20,40\}$, which gives 60 settings
(Appendix~\ref{app:l3}).
We report the cap term of the regret: the regret from the first epoch
of each test batch and from the transition to the final file.

\begin{figure*}[t]
\centering
\includegraphics[width=\textwidth]{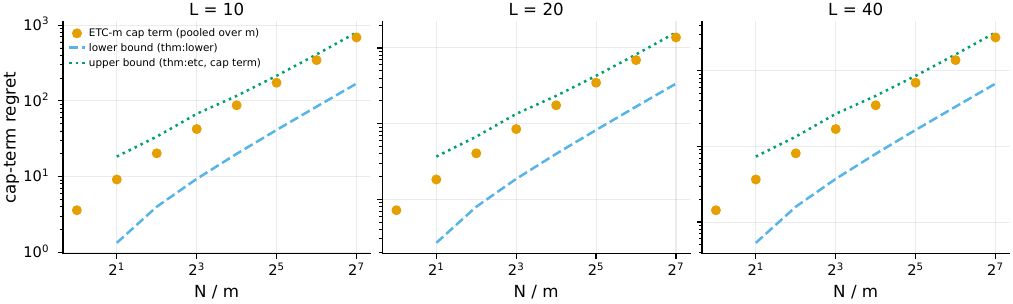}
\caption{
Regret from the edit limit against $N/m$.
Both axes use logarithmic scales.
Each panel uses a different epoch length $L$.
Each point combines simulation settings with the same $N/m$.
The lines show the lower bound from Theorem~\ref{thm:lower}
and the cap term of the upper bound from Theorem~\ref{thm:etc}.
The variation across settings with the same $N/m$ is smaller
than the plotted marker.
}
\label{fig:l3}
\end{figure*}

At each epoch length, the log-log slope of the cap term against
$N/m$ is $1.06$, with a 95\% interval of $[1.03,1.09]$ (Figure~\ref{fig:l3}).
For every setting with $N\ge4m$, the ratio of the cap term to the
lower bound is between $4.1$ and $5.4$, and the cap term stays below
the upper bound in all 60 settings.
The edit-cap term is therefore of order $(N/m)L$, and the bounds of
Theorems~\ref{thm:etc} and~\ref{thm:lower} differ by a constant
factor.